\documentclass[letterpaper]{article} % DO NOT CHANGE THIS
\usepackage{aaai2026}  % DO NOT CHANGE THIS
\usepackage{times}  % DO NOT CHANGE THIS
\usepackage{helvet}  % DO NOT CHANGE THIS
\usepackage{courier}  % DO NOT CHANGE THIS
\usepackage[hyphens]{url}  % DO NOT CHANGE THIS
\usepackage{graphicx} % DO NOT CHANGE THIS
\usepackage{natbib}  % DO NOT CHANGE THIS AND DO NOT ADD ANY OPTIONS TO IT
\usepackage{caption} % DO NOT CHANGE THIS AND DO NOT ADD ANY OPTIONS TO IT
\usepackage{algorithm}
\usepackage{algorithmic}

\usepackage{newfloat}
\usepackage{listings}
\DeclareCaptionStyle{ruled}{labelfont=normalfont,labelsep=colon,strut=off} % DO NOT CHANGE THIS
\floatstyle{ruled}
\newfloat{listing}{tb}{lst}{}
\floatname{listing}{Listing}
\usepackage{booktabs}
\usepackage{amsfonts}
\usepackage{amsmath}
\usepackage{amsthm}
\usepackage{subcaption}
\usepackage{enumitem}
\usepackage{thmtools}
\newtheorem{assumption}{Assumption}
\newtheorem{definition}{Definition}
\newtheorem{lemma}{Lemma}
\newenvironment{proofsketch}[1][Proof Sketch]{\par\noindent\textit{#1.} }{\hfill$\square$\par}
 
\newtheorem{theorem}{Theorem}
\newtheorem{corollary}{Corollary}
\newtheorem*{remark}{Remark}
\usepackage{cleveref}
\usepackage{xcolor}
\usepackage{subcaption}
\title{Enhancing DPSGD via Per-Sample Momentum and Low-Pass Filtering}
\author {
    Xincheng Xu\textsuperscript{\rm 1},
    Thilina Ranbaduge\textsuperscript{\rm 2},
    Qing Wang\textsuperscript{\rm 1},
    Thierry Rakotoarivelo\textsuperscript{\rm 2},
    David Smith\textsuperscript{\rm 2}
}
\affiliations {
    \textsuperscript{\rm 1}School of Computing, Australian National University\\
    \textsuperscript{\rm 2}Data 61, CSIRO\\
    xincheng.xu@anu.edu.au, thilina.ranbaduge@data61.csiro.au, qing.wang@anu.edu.au, thierry.rakotoarivelo@data61.csiro.au, david.smith@data61.csiro.au
}

\usepackage{bibentry}
\begin{document}

\maketitle

\begin{abstract}
Differentially Private Stochastic Gradient Descent (DPSGD) is widely used to train deep neural networks with formal privacy guarantees. However, the addition of differential privacy (DP) often degrades model accuracy by introducing both noise and bias. Existing techniques typically address only one of these issues, as reducing DP noise can exacerbate clipping bias and vice-versa. In this paper, we propose a novel method, \emph{DP-PMLF}, which integrates per-sample momentum with a low-pass filtering strategy to simultaneously mitigate DP noise and clipping bias. Our approach uses per-sample momentum to smooth gradient estimates prior to clipping, thereby reducing sampling variance. It further employs a post-processing low-pass filter to attenuate high-frequency DP noise without consuming additional privacy budget. We provide a theoretical analysis demonstrating an improved convergence rate under rigorous DP guarantees, and our empirical evaluations reveal that DP-PMLF significantly enhances the privacy-utility trade-off compared to several state-of-the-art DPSGD variants.    
\end{abstract}


\section{Introduction}
Deep learning has achieved remarkable success in various domains, such as medical diagnosis~\citep{aggarwal2021, chen2022}, recommendation systems~\citep{chen2023a,fu2023}, and autonomous driving~\citep{bachute2021, chib2023}. However, training deep models often requires large amounts of sensitive data, raising privacy concerns. Recent research has shown that trained models not only could reveal the presence of individuals in a dataset~\citep{choquette2021, olatunji2021}, but are also vulnerable to model inversion or reconstruction attacks~\citep{zhao2021, wang2021a, nguyen2023}. 
%These vulnerabilities highlight the pressing need for training methods that robustly secure sensitive data without compromising model performance.

\textit{Differential Privacy (DP)}~\citep{dwork2014} has become the de facto standard for privacy-preserving deep learning~\citep{tanuwidjaja2020,boulemtafes2020}, offering formal privacy guarantees for training data. Among various DP training algorithms, \textit{Differentially Private Stochastic Gradient Descent (DPSGD)}~\citep{abadi2016} is one of the most widely used methods for training deep neural networks with privacy guarantees. DPSGD enforces ($\epsilon$,$\delta$)-DP through two key mechanisms: (1) \emph{gradient clipping}, which bounds the $\ell_2$ norm of individual gradients to limit the influence of any single training sample, and (2) \emph{noise injection}, where Gaussian noise calibrated to the privacy budget $\epsilon$ and failure factor $\delta$ is added to the aggregated gradients.  However, DPSGD faces a challenging privacy-utility trade-off: per-sample gradient clipping can impede convergence, and the added noise can significantly degrade model performance~\citep{fang2023}. 
%This trade-off often leads to a substantial reduction in model utility, making privacy-preserving deep learning a difficult problem to solve.

To improve the utility of DPSGD, recent works have proposed various strategies, such as adaptively adjusting the clipping threshold~\citep{andrew2021,bu2024,xia2023}, dynamically allocating the privacy budget~\citep{lee2018,yu2019,chen2023b}, projecting gradients into low-dimensional spaces~\citep{zhou2020, yu2021a, yu2021b, asi2021, yu2021c}, designing models less sensitive to DP noise~\citep{papernot2021,wang2021b,shamsabadi2023}, and incorporating public data~\citep{li2022, amid2022, golatkar2022}. Despite these advances, these methods face practical challenges: some lack rigorous theoretical guarantees, others are limited to specific model architectures, or require access to public data, all of which hinder the feasibility of DPSGD in real-world applications.

Beyond these practical challenges, a key theoretical concern is the convergence behavior of DPSGD, which is influenced by two factors: DP noise and clipping bias. Generally, selecting a smaller clipping threshold reduces injected DP noise, minimizing its scale but increasing the clipping bias. Conversely, a larger clipping threshold lowers clipping bias but requires injecting more DP noise to maintain privacy guarantees, which potentially leads to significant performance degradation. 

Many existing methods attempt to mitigate one effect at the expense of the other. For example, Zhang~\textit{et al.}~\citep{zhang2024} applies a low-pass filter to separate DP noise from the true gradient signal, but introduces an additional bias term in the convergence rate. DiceSGD~\citep{zhang2023} utilizes an error-feedback mechanism to correct bias but needs additional DP noise to protect the residual gradient information. Recent studies~\citep{koloskova2023a, xiao2023} suggest that clipping bias is not strictly related to clipping threshold, but is also influenced by sampling variance $\sigma_{SGD}$. This insight inspired our work, as it suggests that it is possible to simultaneously reduce both DP noise and clipping bias, thereby enhancing the overall utility of DPSGD.

In this work, we propose a novel method, \textit{DP-PMLF}, which mitigates both clipping bias and DP noise in DPSGD by integrating per-sample momentum and a low-pass filter. First, per-sample momentum is employed to average historical gradients, thereby reducing the sampling variance and bias introduced by gradient clipping. Second, a low-pass filter is applied as a post-processing step to suppress high-frequency DP noise while preserving the essential low-frequency gradient signal. Our theoretical analysis illustrates an improved convergence guarantee compared to DPSGD under some assumptions, while providing strong privacy guarantees. Empirical results demonstrate our method outperforms resent state-of-the-art techniques. Our main contributions are threefold:

%\vspace{0.2cm}
%\noindent\textbf{Contributions.~}Our main contributions are threefold:
\begin{itemize}
    \item We propose a novel DPSGD method that simultaneously addresses DP noise and clipping bias through the integration of per-sample momentum and low-pass filtering. To the best of our knowledge, our approach is the first to consider reducing the effect of DP noise and clipping bias simultaneously.
    \item We theoretically prove DP-PMLF achieves faster convergence compared to the vanilla DPSGD, while maintaining a mathematically proven privacy guarantee.
    \item Empirical results on different benchmarks demonstrate that our approach achieves a better privacy-utility trade-off compared to various existing state-of-art DPSGD variants across different models and privacy levels. 
\end{itemize}

\iffalse
The remainder of this paper is organized as follows. Section~\ref{sec.rw} reviews related work to improve DPSGD. Section~\ref{sec.p} provides preliminaries on empirical risk minimization, differential privacy and DPSGD. Section~\ref{sec.tpm} introduces two existing state-of-the-art methods and their limitations. Section~\ref{sec.m} presents our methodology and algorithm design. Section~\ref{sec.ta} provides theoretical analysis of convergence and privacy guarantee. Section~\ref{sec.er} presents extensive experimental results supporting the performance of our approach. Finally, Section~\ref{sec.c} concludes with a discussion and future directions.
\fi

\section{Related Work}
\label{sec.rw}
Following the foundational DPSGD work by Abadi~\textit{et al.}~\citep{abadi2016}, 
existing variants on DPSGD can be broadly categorized into two directions: DP noise reduction and clipping bias reduction.

\textbf{DP Noise Reduction: }
To mitigate the effect of DP noise, existing approaches commonly use the following techniques: adaptive clipping threshold~\citep{andrew2021,bu2024,xia2023}, privacy budget allocation~\citep{lee2018,yu2019,chen2023b}, low-rank projection~\citep{zhou2020, yu2021a, yu2021b, asi2021, yu2021c}, specific model design~\citep{papernot2021,wang2021b,shamsabadi2023}, public data assistant~\citep{li2022, amid2022, golatkar2022}.

However, these methods often lack theoretical guarantees, have limited applicability to specific model architectures, or require access to public data for training. To solve these limitations, Zhang~\textit{et al.}~\citep{zhang2024} proposed the introduction of a low-pass filter as a post-processing step in DP optimizers. They demonstrated that low-pass filtering effectively suppresses high-frequency DP noise while preserving essential gradient information.

\textbf{Clipping Bias Reduction: }
%Recent studies have identified that the performance of DPSGD is influenced not only by the injected DP noise to preserve privacy, but also by the bias introduced through gradient clipping operations. 
Koloskova~\textit{et al.}~\citep{koloskova2023a} demonstrated that DPSGD converges with a constant bias term, irrespective of the chosen clipping threshold or the learning rate. %Consequently, DPSGD can only converge to a neighborhood around the optimal solution. 
Chen~\textit{et al.}~\citep{chen2020} also proposed a geometric analysis framework that quantifies gradient clipping bias by measuring the disparity between gradient distributions and symmetric distributions, and a technique to add Gaussian noise to gradients before the clipping operation when gradients are highly asymmetric.

Xiao~\textit{et al.}~\citep{xiao2023} found that the clipping bias is proportional to the sampling variance $\sigma_{SGD}$. The authors proposed to reduce the clipping bias using inner-outer momentum, enhanced network normalization, batch clipping with public data, and data pre-processing. Zhang~\textit{et al.}~\citep{zhang2023} introduced an error-feedback mechanism, \textit{DiceSGD}, to accumulate the difference between clipped and unclipped gradients but requires more DP noise than vanilla DPSGD. 
To avoid clipping operation, Bethune~\textit{et al.}~\citep{bethune2023} proposed \textit{Clipless DPSGD}. This method utilizes Lipschitz-constrained neural networks, which analytically compute gradient sensitivity bounds using projection operations and gradient norm-preserving networks with orthogonal weights. However, \textit{Clipless DPSGD} relies on specific model architectures which limit its practical deployment. 

%We observed that prior work has focused on a single aspect of improving DPSGD performance, either reducing differential privacy noise or mitigating clipping bias. 
%In contrast, our proposed method, DP-PMLF, simultaneously addresses both challenges. We provide robust theoretical guarantees accompanied by improved empirical results. Moreover, our approach can be applied to any deep learning model without the need for public dataset assistance.

\section{Preliminaries}
\label{sec.p}

%This section introduces the concepts used in our work. 
%Emprirical Risk Minimization (ERM) is a key problem in the training of deep neural networks, and Stochastic Gradient Descent (SGD) is an approach to address it. Private variants of SGD, such as Differentially Private SGD (DPSGD), introduce noise into the ERM optimization process, leading to decreased model utility.
% In this section, we introduce the key concepts and notation used throughout our work.
%
We begin by outlining the Empirical Risk Minimization (ERM) problem along with its standard assumptions, and then provide an overview of DP.
%\smallskip
%Table~\ref{tab.n} summarizes the notation, divided into three parts: basic symbols, algorithm-specific terms, and additional notations for our theoretical analysis.

\noindent
\subsection{Empirical Risk Minimization (ERM):}
In this paper, we focus on differentially private optimization developed within the Empirical Risk Minimization (ERM) framework, which forms the basis for supervised deep learning. 
%Incorporating noise into the ERM optimization process to ensure differential privacy compromises model utility. Our approach aims to achieve a better privacy-utility trade-off.
%The definition of ERM is as follows.
Let $D$ be a dataset of $n$ samples, where each sample $\xi$ is drawn from some underlying distribution. In empirical risk minimization (ERM), we seek a parameter vector $x \in \mathbb{R}^d$ that minimizes the average loss:
\begin{equation}\label{eq:erm}
    \min_{x \in \mathbb{R}^d} f(x) \quad \text{with} \quad f(x) = \frac{1}{n} \sum_{\xi \in D} f(x,\xi),
\end{equation}
where $f(x,\xi)$ is the loss incurred on sample $\xi$. 

For clarity, we denote by $\nabla f^{(\xi)}(x) \equiv \nabla_x f(x,\xi)$ the gradient with respect to \(x\) computed on a single sample \(\xi\), \(\|\cdot\|\) denotes the Euclidean norm on \(\mathbb{R}^d\), and \(T\) denote the total number of iterations of the optimization algorithm. We then introduce the following assumptions used in our work:

\begin{assumption}[$L$-Smoothness]\label{ass.lsmo}
A differentiable function $f : \mathbb{R}^d \to \mathbb{R}$ is said to be $L$-smooth if 
it satisfies, for all $x, y \in \mathbb{R}^d$:
\begin{equation*}
    \|\nabla f(x) - \nabla f(y)\| \leq L \|x - y\|.
\end{equation*}
\end{assumption}

\begin{assumption}[Bounded Variance]\label{ass.var}
The per-sample gradient has bounded variance, i.e.,
\begin{equation*}
    \mathbb{E}\left[\|\nabla f^{(\xi)}(x) - \nabla f(x)\|^2\right] \leq \sigma_{SGD}^2, \quad \forall x \in \mathbb{R}^d.
\end{equation*}

\end{assumption}
Here, the expectation is taken with respect to the sampling of \(\xi\) and \(\sigma_{SGD}\) is a constant representing the variance bound.

\begin{assumption}[Bounded Gradient]\label{ass.bod}
The per-sample gradient has a bounded norm, i.e.,
\begin{equation*}
    \|\nabla f^{(\xi)}(x)\| \le G, \quad \forall x \in \mathbb{R}^d, \; \xi \in D,
\end{equation*}
\end{assumption}
where $G$ is a positive constant.

\begin{assumption}[Gradient Auto-Correlation]\label{ass.corr}
For all $t \in \{0, \dots, T-1\}$, there exist sequences $\{c_r\}$ and $\{c_{-r}\}$ with $c_r\geq 0$, and $\forall r \geq 0$,  such that  
\begin{align*}
&\langle \nabla f(x_t), \nabla f(x_{t-r}) \rangle \geq c_r \|\nabla f(x_t)\|^2 + c_{-r} \|\nabla f(x_{t-r})\|^2. 
\end{align*}
\end{assumption}

\begin{assumption}[Independent Sampling Noise]\label{ass.ind}
Let $\zeta_i^{(\xi)} =  \nabla f^{(\xi)}(x_{i}) - \nabla f(x_{i})$ represent the sampling noise from the sample $\xi$ in the $i$-th iteration. If $i \neq j$, then the following condition holds:
    \begin{equation*}
        \mathbb{E}\left[\left(\zeta_i^{(\xi)}\right)^T \zeta_{j}^{(\xi)}\right] = 0.
    \end{equation*}
\end{assumption}

Assumption~\ref{ass.lsmo} is a widely adopted smoothness condition in non-convex optimization~\citep{zaheer2018}. Assumption~\ref{ass.var} is standard in the analysis of gradient clipping~\citep{gorbunov2020}. Assumption~\ref{ass.bod} is commonly used in the DPSGD setting to control the additional bias introduced by clipping~\citep{zhang2023}. Assumptions~\ref{ass.corr} and~\ref{ass.ind} are proposed and validated in~\citep{zhang2024} and~\citep{xiao2023}, respectively.

\subsection{Differential Privacy (DP)}
DP~\citep{dwork2014} provides a privacy guarantee such that the outputs of a mechanism cannot be distinguished by the inclusion or exclusion of any single record in a dataset. Formally, DP is defined as follows:

\iffalse
\begin{definition}[Neighboring datasets]
\label{def.nd}
Two datasets \(D\) and \(D'\) are said to be \emph{neighboring datasets}, denoted by \(D \sim D'\), if \(D'\) can be obtained from \(D\) by either adding or removing exactly one record. 
\end{definition}
\fi

\begin{definition}[Differential Privacy (DP)~\citep{dwork2006}]
\label{def.cdp}
A randomized algorithm $\mathcal{M}: \mathcal{D} \to \mathcal{R}^d$ is \emph{$(\epsilon, \delta)$-DP} if for all neighboring datasets 
$D$ and $D'$, and for any output set $\mathcal{S} \subseteq \mathcal{R}^d$, we have
\begin{equation}\label{eq.DP}
    Pr[\mathcal{M}(D) \in \mathcal{S}] \leq e^\epsilon Pr[\mathcal{M}(D') \in \mathcal{S}] + \delta,
\end{equation}
where $\delta \in [0, 1]$ denotes a failure probability. 
\end{definition}

When $\delta=0$, the mechanism $\mathcal{M}$ is said to satisfy \emph{pure DP}; if $\delta>0$, it satisfies \emph{approximate DP}.

The Gaussian mechanism is widely used to achieve the DP guarantee. The definition of global sensitivity and the Gaussian mechanism are defined as follows:

\begin{definition}[Global Sensitivity]
Let $\mathcal{H}: \mathcal{D} \rightarrow \mathcal{R}^d$ be a function that maps datasets to $d$-dimensional vectors. The global sensitivity of $\mathcal{H}$ is defined as:
$$\Delta \mathcal{H} = \max_{D \sim D'} \| \mathcal{H}(D) - \mathcal{H}(D') \|.$$
\end{definition}

\begin{definition}[Gaussian Mechanism~\citep{dwork2014}]
\label{def.gaussian}
For a function \(\mathcal{H}: \mathcal{D} \to \mathcal{R}^d\) with \(\ell_2\) global sensitivity \(\Delta \mathcal{H}\), the Gaussian mechanism is defined as
\[
\mathcal{M}(D) = \mathcal{H}(D) + \mathcal{N}(0, \sigma_{DP}^2 I_d),
\]
where \(\mathcal{N}(0, \sigma_{DP}^2 I_d)\) denotes the \(d\)-dimensional multivariate Gaussian distribution with mean zero and covariance matrix \(\sigma_{DP}^2 I_d\). The noise parameter is set to
\[
\sigma_{DP} = \frac{\Delta \mathcal{H} \sqrt{2 \ln(1.25/\delta)}}{\epsilon},
\]
which ensures that \(\mathcal{M}\) satisfies \((\epsilon,\delta)\)-DP.
\end{definition}

Our approach is also built upon post-processing, one fundamental DP properties:
\begin{lemma}[Post-Processing~\citep{dwork2006}]
\label{pr.pp}
Let \(\mathcal{M}: \mathcal{D} \to \mathcal{R}^d\) be an \((\epsilon,\delta)\)-DP mechanism and let \(\mathcal{H}: \mathcal{R}^d \to \mathcal{R}^d\) be any deterministic or randomized function. Then the composition \(\mathcal{H} \circ \mathcal{M}\) satisfies \((\epsilon,\delta)\)-DP.
\end{lemma}

\section{Motivation}
\label{sec.tpm}
%\x{Refer to motivation in https://arxiv.org/pdf/2212.00328}.
% \begin{figure}[t!]
%     \centering
%     \includegraphics[width=0.23\textwidth]{./img/example1.pdf}
%     \includegraphics[width=0.23\textwidth]{./img/example2.pdf}
%     %\begin{subfigure}[b]{0.46\textwidth}
%     %    \centering
%     %    \caption{DPSGD and LP-DPSGD with $\epsilon = 8$.}
%     %    \label{fig.example1}
%     %\end{subfigure}
%     %\hfill
%     %\begin{subfigure}[b]{0.46\textwidth}
%     %    \centering
%     %    \caption{DPSGD and InnerOuter with $\epsilon = 1$.}
%     %    \label{fig.example2}
%     %\end{subfigure}
%     \caption{Test accuracy (\%) of LP-DPSDG (left) and InnerOuter (right) 
%     against DPSGD on CIFAR-10 with a 5-Layer CNN over 100 epochs with $\epsilon=8$ and 
%     $\epsilon=1$, respectively.}
%     \label{fig.example}
% \end{figure}

\begin{figure}[t]
    \centering
    \begin{subfigure}[b]{0.3\textwidth}
        \centering
        \includegraphics[width=\textwidth]{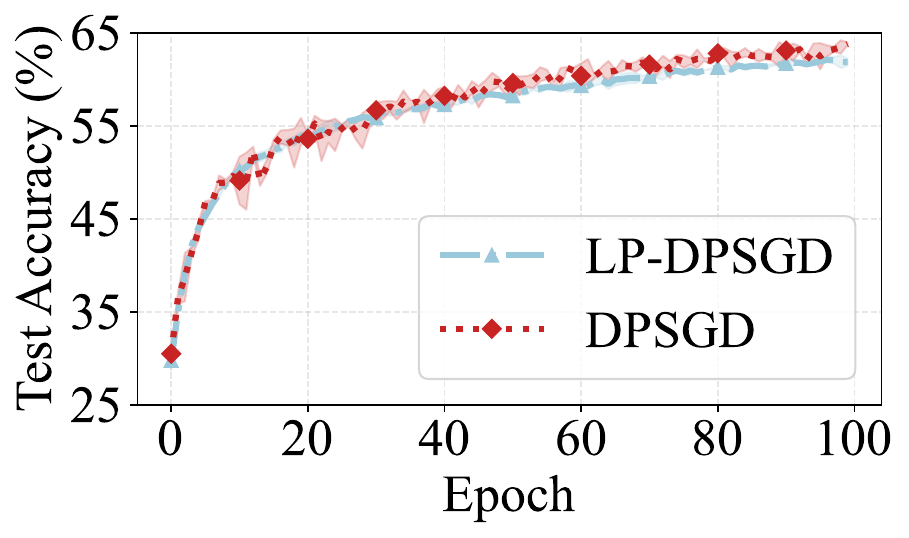}
        \caption{DPSGD and LP-DPSGD with $\epsilon = 8$.}
        \label{fig.example1}
    \end{subfigure}
    \hfill
    \begin{subfigure}[b]{0.3\textwidth}
        \centering
        \includegraphics[width=\textwidth]{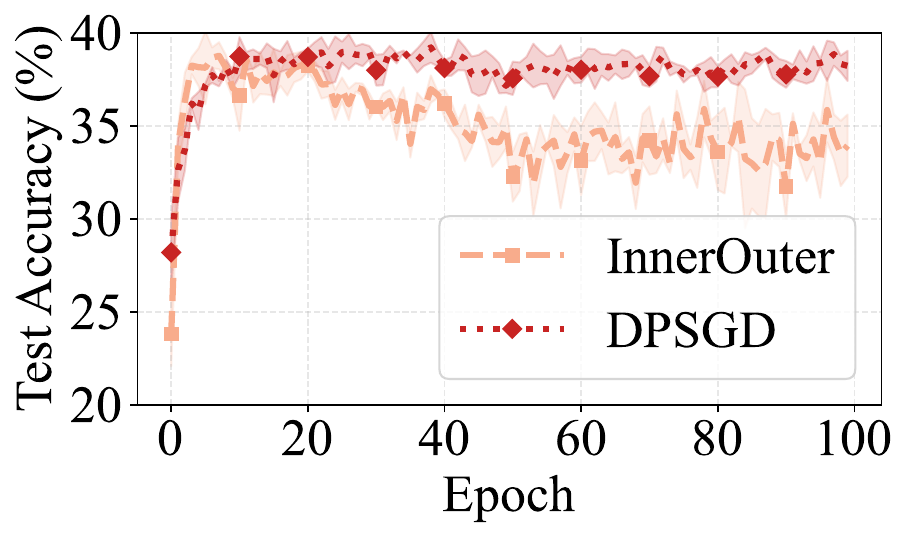}
        \caption{DPSGD and InnerOuter with $\epsilon = 1$.}
        \label{fig.example2}
    \end{subfigure}
    \caption{Test accuracy (\%) comparison of DPSGD and two existing methods on CIFAR-10 with a 5-Layer CNN over 100 epochs under different privacy budgets ($\epsilon$).}
    \label{fig.example}
\end{figure}
%In this section, we introduce two existing methods that address distinct challenges in DPSGD - one aimed at reducing DP noise and the other at mitigating clipping bias. We then discuss the limitations of each method, noting that in some scenarios, their performance can be even worse than DPSGD. 

%\vspace{0.2cm}
\noindent\textbf{Reducing DP Noise.} As mentioned in Section~\ref{sec.rw},
different works have been proposed to mitigate the effects of DP noise. 
The state-of-the-art work by Zhang et al.~\cite{zhang2024}, 
named \textit{LP-DPSGD}, was 
proposed to preserve the integrity of the true gradient signals while mitigating the impact of DP noise. 
%The authors observed that in the frequency domain, true gradient signals are primarily concentrated in the low-frequency range, whereas DP noise is uniformly distributed across all frequency components. Leveraging this observation, 
LP-DPSGD employs a low-pass filter to process gradients, effectively retaining 
low-frequency gradient signals while suppressing high-frequency noise, there by 
improving the signal-to-noise ratio of the gradients.

Although LP-DPSGD reduces the impact of DP noise, it introduces an increased bias as a trade-off. As shown by Koloskova et al.~\cite{koloskova2023a}, if the true gradient is sufficiently large, it can offset the sampling variance $\sigma_{SGD}$, thus preventing clipping bias. However, when incorporating a low-pass filter, this clipping bias cannot be eliminated, as the true gradients from different training iterations could be (negatively) correlated or even completely uncorrelated. 
%As shown in Appendix~\ref{app.dd}, compared to DPSGD, LP-DPSGD introduces an additional constant bias term to $\mathbb{E}[\|\nabla f(x_t)\|^2]$. 

As can be seen in Figure~\ref{fig.example} (a), the performance of LP-DPSGD is even worse than that of vanilla DPSGD. This is because the impact of clipping bias outweighs that of DP noise in this scenario. As a result, while LP-DPSGD effectively suppresses DP noise, the additional bias introduced by the low-pass filter undermines the overall performance. Further details on our analysis are provided in Section~\ref{sec.ca}.

%\vspace{0.2cm}
\noindent\textbf{Mitigating Clipping Bias.~} Xiao~\textit{et al.}~\cite{xiao2023}  found that the clipping bias term is proportional to the sampling variance $\sigma_{SGD}$. 
%Therefore, drawing inspiration from SGD with momentum, 
The authors proposed the \textit{DPSGD with Inner-Outer Momentum} (abbreviated to \textit{InnerOuter} for simplicity) approach as a way of reducing clipping bias. The inner momentum smooths the gradients at the sample level by averaging the gradients in the previous training iterations before clipping, effectively reducing the impact of sampling noise. The outer momentum aggregates the clipped gradients of all samples in a batch, adds DP noise, and applies a second round of smoothing at the batch level.

However, InnerOuter does not perform well when the DP noise is large. The accumulation of historical gradients through the outer momentum operation cannot separate the true gradient and DP noise signals. This accumulates more DP noise which leads to 
inaccurate gradient estimates in training iterations. As seen in 
Figure~\ref{fig.example} (b), InnerOuter is not effective where DP noise dominates the signal. Furthermore, the authors did not provide theoretical proofs leaving InnerOuter's convergence insufficiently validated.

\section{Our Proposed Approach}
\label{sec.m}
%In this section, we first present our key ideas for addressing the challenges of DP noise and clipping bias in DPSGD. Then, we propose a novel algorithm that integrates these ideas in a mathematically rigorous manner.
Now we provide details of DP-PMLF. Our approach is built on two complementary ideas. %\begin{itemize}[leftmargin=4.5mm] \item 
\textbf{(1) Per-sample Momentum:} By maintaining a momentum term for each sample, we average historical gradients over a window of $k$ iterations. This \emph{per-sample momentum} reduces sampling variance and mitigates clipping bias by smoothing out fluctuations before the clipping step. 
%\item 
\textbf{(2) Low-pass Filter:} DP noise is evenly distributed among all frequency components, while true gradient signals concentrate in low frequencies. By applying a linear low-pass filter to the aggregated noisy momentum, we suppress high-frequency noise while preserving the true low-frequency components. 
%\end{itemize}

Together, these ideas balance the trade-off between reducing noise and controlling clipping bias. The per-sample momentum provides a more stable gradient estimate prior to clipping, and the low-pass filter further cleans the aggregated signal without consuming additional privacy budget.

%\vspace{0.2cm}
%\noindent\textbf{Algorithm Description.~} 
Our proposed method is outlined in Algorithm~\ref{alg}. The algorithm proceeds as follows:

\noindent\textbf{Per-sample Momentum Calculation (lines 1 and 5):} For each sample $\xi$, we compute a momentum term by averaging its gradients over the previous $k$ iterations using exponential decay weights. The momentum term
      \[
    v_{t}^{(\xi)} = \sum_{i=t-k+1}^{t} \hat{\beta}^{\,t-i} \nabla f^{(\xi)}(x_i). 
    \]
Here, $\hat{\beta}^{\,t-i} = \frac{\beta^{t-i}}{c_\beta} $ and $ c_\beta = \sum_{i=t-k+1}^{t} \beta^{t-i}.$ $c_\beta$ is the normalization constant that ensures that the momentum coefficients sum to one, preventing excessive accumulation of DP noise.

\noindent\textbf{Momentum Clipping and Noise Addition (lines 6 and 8):}   
 Each sample's momentum $v_{t}^{(\xi)}$ is clipped to a threshold $C$, which bounds the global sensitivity 
    %\[
    $\tilde{v}_t^{(\xi)} = \text{clip}(v_t^{(\xi)}, C)$.
    %\]
    The aggregated momentum is then computed and Gaussian noise with scale $\sigma_{DP}$ is added to each dimension of the average clipped momentum to satisfy DP guarantees:
    \[
    \bar{v}_t = \frac{1}{B} \sum_{\xi \in \mathcal{B}_t} \tilde{v}_t^{(\xi)} + w_t, \quad \text{with } w_t \sim \mathcal{N}(0,\sigma_{DP}^2 I_d).
    \]

%\item 
\noindent\textbf{Low-pass Filtering and Bias Correction (lines 9 - 11):}  
    We apply a linear low-pass filter with coefficients $\{a_r\}_{r=1}^{n_a}$ and $\{b_r\}_{r=0}^{n_b}$ to the aggregated noisy momentum:
    \[
    m_t = -\sum_{r=1}^{n_a} a_r\, m_{t-r} + \sum_{r=0}^{n_b} b_r\, \bar{v}_{t-r},
    \]
    where $m_t$ is the filtered output at time $t$, $\bar{v}_{t-r}$ represents the aggregated noisy momentum at time $t-r$, $\{a_r\}$ and $\{b_r\}$ are the filter coefficients and $n_a$ and $n_b$ determine the filter order. To ensure that after filtering the mean of the signal remains unchanged~\citep{winder2002}, the design of the filter coefficients should satisfy the following constraint:
    \begin{equation}\label{eq.constraint}
        -\sum_{r=1}^{n_a} a_r\ + \sum_{r=0}^{n_b} b_r = 1.
    \end{equation}
    
    We calculate an initialization bias correction term via
    \[
    c_{m,t} = -\sum_{r=1}^{n_a} a_r\, c_{m,t-r} + \sum_{r=0}^{n_b} b_r c_{b,t-r}.
    \]
    We normalize $m_t$ to correct the initialization bias for the filter's effect
    %\[
    $\hat{m}_t = {m_t}/{c_{m,t}}$.
    %\]
    This step smooths the signal, suppressing high-frequency DP noise while retaining the low-frequency, true momentum components.

\noindent\textbf{Model Update (line 12):}  
    Finally, the model parameters are updated using the corrected momentum
    %\[
    $x_{t+1} = x_t - \eta\, \hat{m}_t$.
    %\]
%\end{enumerate}

%Experimental results in Section~\ref{sec.tp} show that 
%the combined use of per-sample momentum and a low-pass filter effectively addresses the limitations identified in prior works. 
%the per-sample momentum smooths the gradients and reduces clipping bias, while the low-pass filter minimizes high-frequency DP noise, leading to cleaner momentum estimates. Importantly, the DP guarantees are maintained: the clipping threshold $C$ bounds the per-sample's global sensitivity, and the noise added is calibrated accordingly. The filter, being a post-processing step, does not affect the overall privacy budget consuming.

\begin{algorithm}[t!]
\caption{DP-PMLF}
\label{alg}
\begin{footnotesize}
\begin{algorithmic}[1]
\REQUIRE dataset $D$, initial model parameters $x_0$, learning rate $\eta$, momentum length $k$, 
         filter parameters $\{a_r\}_{r=1}^{n_a}$, $\{b_r\}_{r=0}^{n_b}$,
         clipping threshold $C$, noise scale $\sigma_{DP}$, batch size $B$, iteration number $T$, per-sample momentum factor $\beta$
\STATE $c_\beta = sum(\sum_{i=t-k + 1}^{t} \beta^{t-i})$
\FOR{$t = 0$ to $T-1$}
    \STATE Sample minibatch $\mathcal{B}_t$ of size $B$ from $D$
    \FOR{$\xi \in \mathcal{B}_t$}
        \STATE $v_{t}^{(\xi)} =  \sum_{i=t-k + 1}^{t} \hat{\beta}^{t-i} \nabla f^{(\xi)}(x_{i}), \text{ where }\hat{\beta}^{t-i} = \frac{\beta^{t-i}}{c_\beta}$
        \STATE $\tilde{v}_t^{(\xi)} = \text{clip}(v_t^{(\xi)}, C)$
    \ENDFOR
    
    \STATE $\bar{v}_t = \frac{1}{B} \sum_{\xi \in \mathcal{B}_t} \tilde{v}_t^{(\xi)}  + w_t, \text{ where } w_t \sim \mathcal{N}(0,\sigma_{DP}^2 I_d)$
    \STATE $m_t = -\sum_{r=1}^{n_a} a_r m_{t-r} + \sum_{r=0}^{n_b} b_r \bar{v}_{t-r}$
    \STATE $c_{b,t} = 1, c_{m,t} = -\sum_{r=1}^{n_a} a_r c_{m,t-r} + \sum_{r=0}^{n_b} b_r c_{b,t-r} $
    \STATE $\hat{m}_t = m_t/c_{m,t}$
    \STATE $x_{t+1} = x_t - \eta \hat{m}_t$
\ENDFOR
\RETURN $x_T$
\end{algorithmic}
\end{footnotesize}
\end{algorithm}

\section{Theoretical Analysis}
\label{sec.ta}
%In this section, we analyze the convergence and privacy guarantees of our proposed algorithm under the ERM setting.

%\subsection{Two Key Lemmas}
We now present key lemmas that underpin our convergence analysis.
%
%\vspace{0.2cm}
%\noindent\textbf{Filter representation and aggregation.~} 
\begin{restatable}[Effectiveness of Low-pass Filter]{lemma}{trans}
\label{lem.trans}
\begin{footnotesize}
\[
\hat{m}_t = \sum_{r=0}^t \hat{\kappa}_r\, \bar{v}_{t-r}, \text{ with}
\]
\[
\hat{\kappa}_r = \frac{\kappa_r}{\sum_{r=0}^t \kappa_r} \quad \text{and} \quad \kappa_r = \sum_{r_2=0}^{\min(n_b, r)} b_{r_2} \sum_{r_1=1}^{n_a} z_{a,r_1}\, (p_{a,r_1})^{r - r_2}.
\]    
\end{footnotesize}

\end{restatable}

We begin by analyzing how the low-pass filter aggregates historical momentum information while attenuating high-frequency DP noise. This is crucial because, as shown in Section~3.3, the true gradient signal concentrates in the low-frequency regime while DP noise is spectrally flat. Thus, the low-pass filter suppresses the high-frequency components, as indicated by the decay properties of \(\{\hat{\kappa}_r\}\). The proof of this lemma is available in the appendix.

\begin{restatable}[Bounded Momentum Variance]{lemma}{bmv}
\label{lem.bmv}
Under Assumptions~\ref{ass.lsmo}, \ref{ass.var}, and \ref{ass.ind}, if the step size satisfies
\begin{footnotesize}
\[
\eta \leq \sqrt{\frac{\sigma_{SGD}^2}{L^2k^3(C^2 + d \sigma_{DP}^2)}}, 
\]
then 
\[
\mathbb{E}\left[\|v_t^{(\xi)} - \nabla f(x_t)\|^2\right] \le \mathcal{O}\left(\frac{\sigma_{SGD}^2}{\rho^2}\right),
\]
where
\[
\rho = \sqrt{\frac{(1 + \beta)(1 - \beta^{k})}{(1-\beta)(1 + \beta^{k})}}.
\]
\end{footnotesize}
\end{restatable}
\begin{proofsketch}
We decompose the error into the variance of the sampling noise
%\[
$\mathbb{E}\left[\|\nabla f^{(\xi)}(x_i) - \nabla f(x_i)\|^2\right]$,
%\]
and the error due to the drift between \(\nabla f(x_i)\) and \(\nabla f(x_t)\). The former is directly controlled by Assumption~\ref{ass.var} and the independence in Assumption~\ref{ass.ind}, while the latter is bounded via the \(L\)-smoothness condition (Assumption~\ref{ass.lsmo}). The weighted averaging in the per-sample momentum reduces the overall variance by the factor \(\rho^2\). Detailed derivations are provided in the appendix.
\end{proofsketch}
\begin{remark}
The factor \(\rho^2\) increases with both the per-sample momentum factor \(\beta\) and the momentum length \(k\). When \(\beta\) reaches its maximum value of 1, the exponential decay reduces to an equal-weight average, and \(\rho^2\) approaches \(k\). Theoretically, larger values of \(\beta\) and \(k\) are preferable; however, in practice, a large \(\beta\) may cause training to rely overly on historical information, potentially slowing convergence.
%, and a large \(k\) increases memory and computational requirements as we show in our experiments.
\end{remark}

\noindent\textbf{Convergence Analysis}
\label{sec.ca}
We now combine the above lemmas to establish the convergence rate of 
Algorithm~\ref{alg}. The analysis builds on a standard descent lemma for \(L\)-smooth functions and is augmented by our representation of the low-pass filtered momentum and the variance reduction effect.

% \vspace{0.2cm}
\noindent\textit{Step 1: Descent Lemma.}  
By \(L\)-smoothness from Assumption~\ref{ass.lsmo}, we have
\begin{footnotesize}
\[
\mathbb{E}\left[f(x_{t+1}) - f(x_t)\right] \leq -\eta\,\mathbb{E}\left[\langle \nabla f(x_t), \hat{m}_t \rangle\right] + \frac{L\eta^2}{2}\, \mathbb{E}\left[\|\hat{m}_t\|^2\right].
\]
\end{footnotesize}

\noindent\textit{Step 2: Decomposition of Gradient.}  
Using the representation from Lemma~\ref{lem.trans}, we write
\[
\langle \nabla f(x_t), \hat{m}_t \rangle = \sum_{r=0}^t \hat{\kappa}_r\, \langle \nabla f(x_t), \bar{v}_{t-r} \rangle.
\]
We decompose each inner product into two parts:
\begin{enumerate}
    \item The correlation between the current gradient and the historical (unclipped) momentum, which under Assumption~\ref{ass.corr} can be bounded in terms of \(\|\nabla f(x_t)\|^2\) and \(\|\nabla f(x_{t-r})\|^2\).
    \item A bias term arising from clipping, i.e., the difference \(\mathbb{E}[\tilde{v}_{t-r}^{(\xi)} - v_{t-r}^{(\xi)}]\), which is bounded by \(\mathcal{O}\left(G + \frac{\sigma_{SGD}}{\rho}\right)\) using Assumption~\ref{ass.bod} and Lemma~\ref{lem.bmv}.
\end{enumerate}
\textit{Step 3: Final Convergence Bound.}  
After careful estimation of the descent term and the term \(\mathbb{E}\left[\|\hat{m}_t\|^2\right]\) (which also incorporates the effect of the DP noise with variance \(d \, \sigma_{DP}^2\)), telescoping over \(T\) iterations and taking averages yields the following convergence guarantee:
\begin{restatable}[Convergence Bound]{theorem}{con}
\label{thm.con}
Under Assumptions~\ref{ass.lsmo}--\ref{ass.ind}, if Algorithm~\ref{alg} is running for \(T\) iterations with step size
\[
\eta \leq \sqrt{\frac{\sigma_{SGD}^2}{L^2k^3(C^2 + d \, \sigma_{DP}^2)}},
\]
then $\mathbb{E}\left[\|\nabla f(x_t)\|^2\right]$ is upper bounded by
\begin{scriptsize}
\[
\mathcal{O}\left(\frac{f(x_0) - f^*}{\eta T}  + L\eta C^2 + \frac{L\eta}{\Gamma_{DP}}\, d\, \sigma^2_{DP} + \left(\frac{G^2}{\Gamma_{SGD}} +\frac{\sigma_{SGD}^2}{\rho^2 \, \Gamma_{SGD}}\right)\right),
\]
\end{scriptsize}
where
$\hat{c}_r = \sum_{i=t-r-k+1}^{t-r} \hat{\beta}^{t-r-i}c_{t-i}, \, \rho = \sqrt{\frac{(1+\beta)(1-\beta^k)}{(1-\beta)(1+\beta^k)}}$, and $f^* = min_x f(x).$ $\Gamma_{DP} = \frac{\sum_{t=0}^{T-1}\sum_{r=0}^t \hat{\kappa}_r \hat{c}_r}{\sum_{t=0}^{T-1}\sum_{r=0}^t \hat{\kappa}_r^2}$ and $\Gamma_{SGD} = \frac{\sum_{t=0}^{T-1}\sum_{r=0}^t \hat{\kappa}_r \hat{c}_r}{\sum_{t=0}^{T-1}\sum_{r=0}^t  \frac{\hat{\kappa}_r}{\hat{c}_{r}}}$ are two ratios introduce by low-pass filtering.
\end{restatable}

\begin{remark}
The first term represents the optimization error decreasing with the number of iterations.
 The second term \(L\eta C^2\) reflects the error introduced by gradient clipping. The third term captures the impact of the DP noise, where the aggregated effect is modulated by the filter coefficients. The final term aggregates the residual bias due to low-pass filtering and the reduced clipping bias from per-sample momentum. 
 %Compared to LP-DPSGD~\citep{zhang2024}, 
 Our approach reduces the clipping bias term by introducing the variance reduction factor \(\rho\).
\end{remark}
\begin{corollary}\label{cor}
If the per-sample gradient norm is bounded by \(G = \mathcal{O}\!\left(\frac{\sigma_{SGD}}{\rho}\right)\), then the bound in Theorem~\ref{thm.con} simplifies to
\[
\mathcal{O}\!\left(\frac{f(x_0)-f^*}{\eta T} + L\eta C^2 + \frac{L\eta\, d\, \sigma_{DP}^2}{\Gamma_{DP}} + \frac{\sigma_{SGD}^2}{\rho^2 \, \Gamma_{SGD}}\right).
\]
\end{corollary}

\begin{remark}
    A careful inspection of our convergence bound reveals that, by choosing \(\beta\), \(k\), and low-pass filter coefficients appropriately, our approach reduces the clipping bias term by introducing the variance reduction factor \(\rho\) and mitigates the effect of DP noise via the low-pass filter compared to vanilla DPSGD~\citep{abadi2016}. 
    %This improvement is achieved without consuming additional privacy budget, since the low-pass filter is a post-processing step.
\end{remark}

\begin{table*}[t!]
\caption{Test accuracy (\%) comparison across datasets on ViT with fixed epoch (Epoch = 25 for MNIST and Fashion-MNIST, Epoch = 50 for CIFAR-10 and CIFAR-100) and different privacy budgets $\epsilon$ = 1 and 8.}
\centering
\resizebox{\textwidth}{!}{%
\begin{tabular}{l|c|c|c|c|c|c|c|c}
\hline
Method & \multicolumn{2}{c|}{MNIST} & \multicolumn{2}{c|}{Fashion-MNIST} & \multicolumn{2}{c|}{CIFAR-10} & \multicolumn{2}{c}{CIFAR-100} \\
\cline{2-9}
& $\epsilon{=}1$ & $\epsilon{=}8$ & $\epsilon{=}1$ & $\epsilon{=}8$ & $\epsilon{=}1$ & $\epsilon{=}8$ & $\epsilon{=}1$ & $\epsilon{=}8$ \\
\hline
DPSGD & 89.00 ± 0.06 & 88.95 ± 0.01 & 78.96 ± 0.05 & 79.04 ± 0.04 & 35.74 ± 0.26 & 47.74 ± 1.20 & 7.52 ± 0.49 & 18.27 ± 0.48 \\
LP-DPSGD & 88.99 ± 0.06 & 88.96 ± 0.01 & 79.03 ± 0.10 & 79.02 ± 0.10 & 35.84 ± 0.63 & 48.37 ± 0.36 & 7.55 ± 0.26 & 18.52 ± 0.27 \\
InnerOuter & 92.15 ± 0.15 & \textbf{92.43 ± 0.06} & 80.50 ±2.28 & 81.18 ± 1.56 & 11.55 ± 1.07 & 33.53 ± 0.52 & 1.13 ± 0.20 & 13.93 ± 0.40 \\
DP-PMLF & \textbf{92.16 ± 0.05} & 92.39 ± 0.07 & \textbf{80.65 ± 1.17} & \textbf{81.93 ± 0.83} & \textbf{40.96 ± 1.18} & \textbf{51.47 ± 0.33} & \textbf{ 11.40 ± 0.21} & \textbf{23.15 ± 0.52} \\
\hline
\end{tabular}
}
\label{tab.vit}
\end{table*}
\begin{figure*}[t!]
    \centering
    \includegraphics[width=0.8\textwidth]{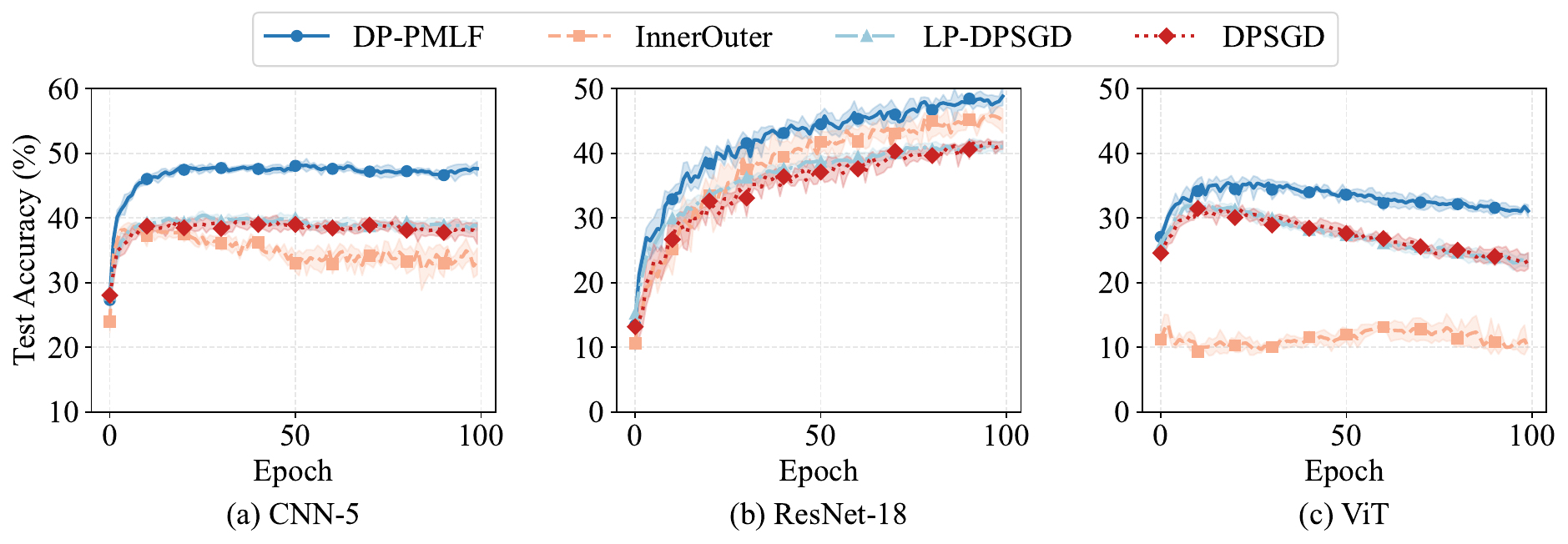}
    \caption{Test accuracy (\%) comparison across different models on CIFAR-10 with fixed privacy budget $\epsilon =1.$}
    \label{fig.model}
\end{figure*}

\noindent\textbf{Privacy Analysis}
We establish that our approach satisfies \((\epsilon,\delta)\)-differential privacy.
\begin{theorem}[Differential Privacy Guarantee]
There exist absolute constants \(c_1\) and \(c_2\) such that, given sampling probability \(q = B/n\) and \(T\) iterations, for any \(\epsilon < c_1 q^2T\), Algorithm~\ref{alg} is \((\epsilon,\delta)\)-differentially private for any \(\delta > 0\) if the noise scale ($\sigma_{DP}$) satisfies
\[
\sigma_{DP} \geq c_2\frac{q\sqrt{T\log(1/\delta)}}{\epsilon}.
\]
\end{theorem}
\begin{proof}
Let \(D\) and \(D'\) be any two neighbouring datasets where \(D'\) contains exactly one additional sample \(\xi\) compared to \(D\). The global sensitivity of the clipped per-sample momentum satisfies
%\[
$\|\text{clip}(v_t^{(\xi)}, C)\| \leq C$.
%\]
Thus, by utilizing the Gaussian mechanism in Definition~\ref{def.gaussian} and privacy amplification by subsampling~\citep{balle2018}, $\bar{v}_t$ in each training iteration is protected by \((\mathcal{O}(q/\sigma_{DP}), \delta/T)\)-DP. Given that the subsequent low-pass filter is applied as a post-processing step (Lemma~\ref{pr.pp}) in each training iteration, $\hat{m}_t$ is also protected by \((\mathcal{O}(q/\sigma_{DP}), \delta/T)\)-DP. The moments accountant method~\citep{abadi2016} then implies that, over \(T\) iterations, the overall privacy guarantee is \((\epsilon,\delta)\)-DP provided if the stated condition on \(\sigma_{DP}\) holds.
\end{proof}

\section{Experiments}\label{sec.er}
%\x{Table 2 \& Figure 4 \& Ablation study. Refer to https://arxiv.org/pdf/2408.13460}
%In this section, we evaluate our proposed DP-PMLF approach through comprehensive experiments. Section~\ref{sec.es} outlines the experimental setup. Section~\ref{sec.tp} shows the trade-off between privacy and utility under different settings. Section~\ref{sec.as} conducts an ablation study to validate the effectiveness of the per-sample momentum and the low-pass filter. Section~\ref{sec.ps} analyzes the effect of choosing hyper-parameters.
Next we evaluate DP-PMLF through comprehensive experiments. 
Due to page limitation, details of the experimental setting and additional results are given in the appendix. 

% For repeatability, we have made the code of our approach DP-PMLF available, which can be accessed through \url{https://anonymous.4open.science/r/DP-PMLF-521F}.  

\subsection{Experiment Setting}\label{sec.es}

\noindent\textbf{Dataset.} We evaluate our approach on four image classification datasets, including MNIST~\citep{deng2012}, Fashion-MNIST~\citep{xiao2017}, CIFAR-10~\citep{krizhevsky2009}, and CIFAR-100~\citep{krizhevsky2009}, and four sentence classification datasets, including MNLI, QNLI, QQP, and SST-2 from the GLUE benchmark~\citep{wang2018}.

\noindent\textbf{Baselines.} We compare the test accuracy of DP-PMLF with vanilla DPSGD~\citep{abadi2016} and two state-of-the-art methods introduced in Section~\ref{sec.tpm}: LP-DPSGD~\citep{zhang2024} and InnerOuter~\citep{xiao2023}. 

\noindent\textbf{Models.} We utilized three models for image classification tasks: a 5-layer CNN~\citep{zhang2024}, ResNet-18~\citep{he2016}, and the Vision Transformer (ViT)~\citep{dosovitskiy2021}. These models are initialized with random weights without pretraining. For sentence classification tasks,  we
fine-tune a pre-trained RoBERTa-base model~\citep{liu2019}.

\noindent\textbf{Hyper-parameters.} The parameter choices are detailed in the appendix and based on settings commonly used in the literature. All experiments are repeated five times, with the mean and standard deviation reported. 
%The parameter choices used in the experiments are listed in Appendix~\ref{sec:expsetup}.

\subsection{Privacy-Utility Trade-off}
\label{sec.tp}
\subsubsection{Image Classification}
We compare the performance of our method against baselines across different 
models and datasets with varying privacy budgets $\epsilon$.
We report the test accuracy for $\epsilon = 1$ and $\epsilon = 8$ for the ViT model 
in Table~\ref{tab.vit}. As can be seen, DP-PMLF maintains its leading performance. 
For instance, on Fashion-MNIST, it achieves accuracies of about 80.65\% at 
\(\epsilon = 1\) and 81.93\% at \(\epsilon = 8\). On CIFAR-100, our approach 
maintains a 4--5\% margin over the next-best baselines across both privacy budgets.

\begin{table*}[t]
\centering
\caption{Test accuracy (\%) comparison across GLUE benchmark subsets with different privacy budgets $\epsilon = 1, 8$.}
\resizebox{\textwidth}{!}{%
\begin{tabular}{l|c|c|c|c|c|c|c|c}
\hline
Method & \multicolumn{2}{c|}{MNLI} & \multicolumn{2}{c|}{QNLI} & \multicolumn{2}{c|}{QQP} & \multicolumn{2}{c}{SST-2} \\
\cline{2-9}
& $\epsilon{=}1$ & $\epsilon{=}8$ & $\epsilon{=}1$ & $\epsilon{=}8$ & $\epsilon{=}1$ & $\epsilon{=}8$ & $\epsilon{=}1$ & $\epsilon{=}8$ \\
\hline
DPSGD & 51.36 ± 0.66 & 72.00 ± 0.23 & 65.59 ± 0.66 & 85.47 ± 0.78 & 71.20 ± 0.97 & 80.38 ± 0.37 & 76.19 ± 1.15 & \textbf{90.83 ± 0.38} \\
LP-DPSGD & 52.75 ± 0.62 & 71.48 ± 0.23 & 66.34 ± 0.94 & 85.44 ± 0.45 & 71.61 ± 0.71 & 80.55 ± 0.40 & 76.46 ± 0.24 & 89.24 ± 0.66 \\
InnerOuter & 48.45 ± 0.89 & 70.35 ± 0.38 & 69.46 ± 0.87 & 86.07 ± 0.58 & 70.27 ± 0.64 & 83.18 ± 0.34 & 76.49 ± 0.63 & 89.08 ± 0.43 \\
DP-PMLF & \textbf{56.81 ± 0.74} & \textbf{75.56 ± 0.42} & \textbf{72.38 ± 0.62} & \textbf{86.96 ± 0.69} & \textbf{75.55 ± 1.16} & \textbf{83.42 ± 0.52} & \textbf{78.07 ± 0.96} & 90.39 ± 1.03 \\
\hline
\end{tabular}
}
\label{tab.glue}
\end{table*}

\begin{figure*}[t]
    \centering
    \includegraphics[width=1.0\textwidth]{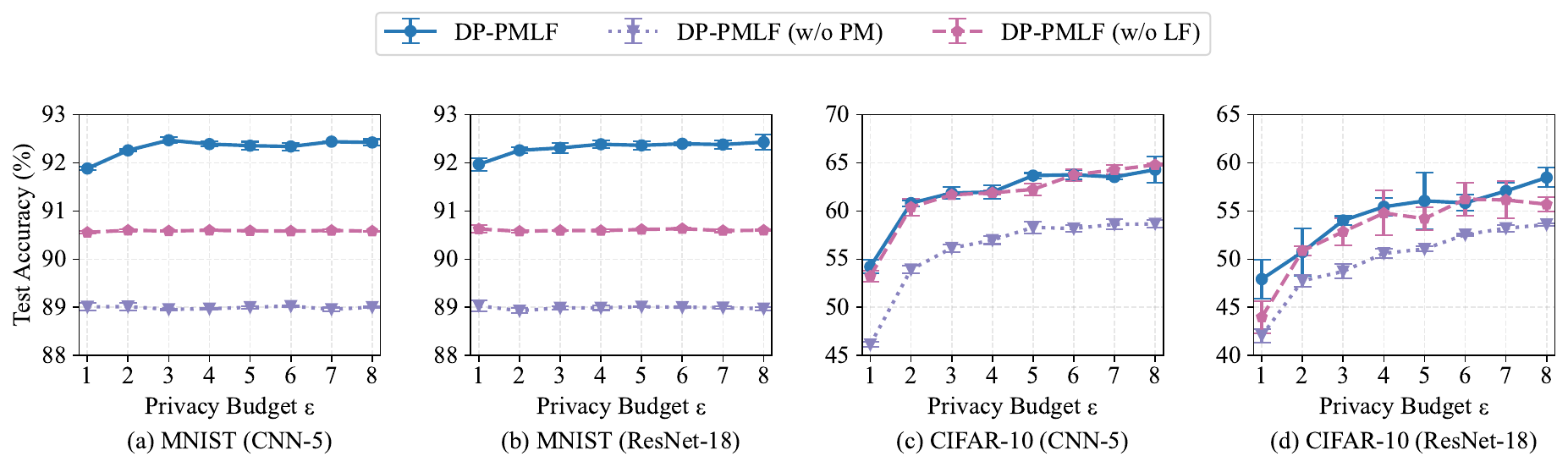}
    \caption{Test accuracy (\%) for DP-PMLF and its two variants, which are DP-PMLF without Per-sample Momentum (DP-PMLF (w/o PM)) and DP-PMLF without Low-pass Filter (DP-PMLF (w/o LF)).}
    \label{fig.ablation}
\end{figure*}

Under a high DP noise regime ($\epsilon = 1$), as shown in Table~\ref{tab.vit}, the
InnerOuter method shows a degradation in performance. This is because the InnerOuter 
method lacks normalization and suffers from excessive noise accumulation. In contrast,
DP-PMLF utilizes both normalization and a low-pass filter allowing our approach to 
better control and filter DP noise, thereby achieving superior results.

Figure~\ref{fig.model} presents the test accuracy on CIFAR-10 under \(\epsilon = 1\) 
for three model architectures: CNN-5, Resnet-18, and ViT. In all cases, DP-PMLF 
consistently surpasses the baseline methods. For example, with CNN-5, DP-PMLF 
attains approximately 47\% accuracy, exceeding the best baseline by around 9\%. 
Similarly, with ResNet-18, DP-PMLF reaches nearly 50\%, which is 1--2\% higher 
than the strongest competitor. Finally, when using ViT, DP-PMLF achieves about 
31\% accuracy, compared to only 23\% for the best baseline.

\subsubsection{Sentence Classification}
To further assess the performance of our approach, we extend our evaluation to sentence classification tasks using four datasets from the GLUE benchmark, with results presented in Table~\ref{tab.glue}. These experiments further demonstrate the effectiveness of DP-PMLF, which consistently shows a significant performance improvement over other baselines. When $\epsilon = 1$, our method surpasses the  baselines by over 4\% on MNLI and nearly 3\% on QNLI. Although the performance gap decreases under a more relaxed privacy budget of $\epsilon = 8$, DP-PMLF still outperforms or remains highly competitive with the baseline methods. These results show the effectiveness of our approach for sentence classification tasks in a differential privacy setting.

\subsection{Ablation Studies}
\label{sec.as}
We evaluated on two core components of DP-PMLF: per-sample momentum and the low-pass filter. Specifically, we denote DP-PMLF without per-sample momentum as \textit{DP-PMLF (w/o PM)} and DP-PMLF without the low-pass filter as \textit{DP-PMLF (w/o LF)}. We used MNIST and CIFAR-10 on the CNN-5 and Resnet-18 models with different privacy budget ($\epsilon$) values ranging from 1 to 8.
Figure~\ref{fig.ablation} shows that DP-PMLF consistently outperforms DP-PMLF (w/o PM) for different $\epsilon$ values. This is because per-sample momentum effectively reduces clipping bias and thus narrows the neighborhood around the optimal convergence point. 

Further, compared to DP-PMLF (w/o LF), our approach exhibits superior performance on MNIST by leveraging historical gradients. This refines the gradient descent direction and potentially accelerates convergence. For CIFAR-10, we adopt more complicated filter coefficients to enhance gradient signal smoothing and mitigate DP noise. This is advantageous when DP noise is large, e.g., when $\epsilon \leq 6.$ However, when DP noise is relatively small ($\epsilon > 6$), excessive smoothing may lead to the loss of true gradient information, causing DP-PMLF to perform slightly worse than DP-PMLF (w/o LF). This is evident in Figure~\ref{fig.ablation}(c) where our approach is achieving approximately 0.5--0.7\% less test accuracy than DP-PMLF (w/o LF) when $\epsilon > 6.$

\section{Conclusion and Future Work}\label{sec.c}
In this work, we propose a novel DPSGD variant that incorporates per-sample momentum 
and a low-pass filter to simultaneously reduce the effect of DP noise and clipping bias. 
We provide a theoretical proof of an improved convergence rate associated with a formal
DP guarantee. Our experimental results show that our approach achieves higher utility
in image and sentence classifications compared to the state-of-the-art DPSGD variants.
In future work, we will investigate how to analyze our approach under some 
general assumptions such as non-convex Polyak-Łojasiewicz conditions~\citep{karimi2016}, 
and $(L_0,L_1)$-smoothness~\citep{zhang2020b}. Furthermore, we will develop an adaptive
method to optimize the selection of the hyper-parameters in per-sample momentum and
low-pass filtering. Finally, we will investigate how to apply this method to different
domain applications, such as natural language processing tasks and reinforcement
learning.

\bibliography{aaai2026}

\appendix
%\section*{Appendix}

\section{Proof of Lemma~\ref{lem.trans}}
\label{app.lem1}
We analyze how the low-pass filter aggregates historical 
momentum information while attenuating high-frequency DP noises by using the 
z-transform and a partial fraction decomposition. 
%Here, we show that applying the z-transform and performing a partial-fraction decomposition yields an expression of the form
%\[
%\hat{m}_t = \sum_{r=0}^t \hat{\kappa}_r\, \bar{v}_{t-r},
%\]
%where the unnormalized coefficients are
%\[
%\kappa_r = \sum_{r_2=0}^{\min(n_b, r)} b_{r_2} \sum_{r_1=1}^{n_a} z_{a,r_1}\, (p_{a,r_1})^{r - r_2},
%\]
%with 
%We recall Lemma~\ref{lem.trans}:
\trans*
Here,  \(z_{a,r_1}\) and \(p_{a,r_1}\) coming from the z-transform decomposition. The normalized coefficients are then defined as
\[
\hat{\kappa}_r = \frac{\kappa_r}{\sum_{r=0}^t \kappa_r}.
\]
Since \(c_{m,t} = \sum_{r=0}^t \kappa_r\), the filter forms a convex combination of past aggregated momenta \(\{\bar{v}_{t-r}\}\).

\begin{proof}
The first part of the proof follows from Appendix B in~\cite{zhang2024}. Here, we detail the derivation of the filter representation.

We start by expanding the recursive definition of \(m_t\):
\begin{align}
m_t &= -\sum_{r=1}^{n_a} a_r\, m_{t-r} + \sum_{r=0}^{n_b} b_r\, \bar{v}_{t-r}. \label{eq:recurrence}
\end{align}
Rearranging the above equation gives:
\begin{align}
m_t + \sum_{r=1}^{n_a} a_r\, m_{t-r} &= \sum_{r=0}^{n_b} b_r\, \bar{v}_{t-r}.
\end{align}
Taking the \(z\)-transform (denoted by \(\mathcal{Z}\{\cdot\}\)) on both sides and using the time-shift property \(\mathcal{Z}\{x_{t-k}\} = z^{-k}X(z)\), we obtain:
\begin{align}
\left(1 + \sum_{r=1}^{n_a} a_r z^{-r}\right) M(z) &= \left(\sum_{r=0}^{n_b} b_r z^{-r}\right) V(z). \label{eq:ztransform}
\end{align}
Dividing both sides of~\Cref{eq:ztransform} by \(1 + \sum_{r=1}^{n_a} a_r z^{-r}\) yields:
\begin{align}
M(z) &= \frac{\sum_{r=0}^{n_b} b_r z^{-r}}{1 + \sum_{r=1}^{n_a} a_r z^{-r}}\, V(z). \label{eq:Mz}
\end{align}
We now define constants \(\{z_{a,r}\}\) and \(\{p_{a,r}\}\) such that
\[
\frac{1}{1 + \sum_{r=1}^{n_a} a_r z^{-r}} = \sum_{r=1}^{n_a} \frac{z_{a,r}}{1 - p_{a,r} z^{-1}}.
\]
Substituting this into~\Cref{eq:Mz} gives:
\begin{align}
M(z) &= \sum_{r_1=1}^{n_a} \frac{z_{a,r_1}}{1 - p_{a,r_1} z^{-1}} \left(\sum_{r_2=0}^{n_b} b_{r_2} z^{-r_2}\right) V(z). \label{eq:step_c}
\end{align}
Next, we expand the term \(\frac{1}{1 - p_{a,r_1} z^{-1}}\) as a power series:
\[
\frac{1}{1 - p_{a,r_1} z^{-1}} = \sum_{r_0=0}^{\infty} \left(p_{a,r_1} z^{-1}\right)^{r_0}.
\]
Substituting this expansion into~\Cref{eq:step_c} yields:
\begin{footnotesize}
\begin{align}
M(z) &= \sum_{r_1=1}^{n_a} z_{a,r_1} \sum_{r_0=0}^{\infty} \left(p_{a,r_1} z^{-1}\right)^{r_0} \left(\sum_{r_2=0}^{n_b} b_{r_2} z^{-r_2}\right) V(z). 
\label{eq:step_d}
\end{align}
\end{footnotesize}
By rearranging the summations, we have:
\begin{align}
M(z) &= \sum_{r_2=0}^{n_b} b_{r_2} \sum_{r_1=1}^{n_a} z_{a,r_1} \sum_{r_0=0}^{\infty} \left(p_{a,r_1}\right)^{r_0} z^{-(r_0+r_2)} V(z). \label{eq:step_e}
\end{align}
Taking the inverse \(z\)-transform of~\Cref{eq:step_e} gives:
\begin{align}
m_t &= \sum_{r_2=0}^{n_b} b_{r_2} \sum_{r_1=1}^{n_a} z_{a,r_1} \, p_{a,r_1}^{\, t - r_2}\, \bar{v}_{t-r_2}, \label{eq:step_f}
\end{align}
where the exponent \(t - r_2\) corresponds to the summation index \(r_0\). We can then define the combined coefficient:
\[
\kappa_r = \sum_{r_2=0}^{\min\{n_b, r\}} b_{r_2} \sum_{r_1=1}^{n_a} z_{a,r_1}\, p_{a,r_1}^{\, r - r_2}.
\]
Thus, the sequence \(\{m_t\}\) can be written as:
\[
m_t = \sum_{r=0}^{t} \kappa_r\, \bar{v}_{t-r}.
\]
Finally, let \(c_{m,t} = \sum_{r=0}^{t} \kappa_r\) be the normalization constant. Then the normalized momentum is given by:
\[
\hat{m}_t = \frac{m_t}{c_{m,t}} = \sum_{r=0}^{t} \frac{\kappa_r}{\sum_{r=0}^{t} \kappa_r}\, \bar{v}_{t-r} = \sum_{r=0}^{t} \hat{\kappa}_r\, \bar{v}_{t-r},
\]
where we define \(\hat{\kappa}_r = \frac{\kappa_r}{\sum_{r=0}^{t} \kappa_r}\).
This completes the proof.
\end{proof}

\section{Proof of Lemma~\ref{lem.bmv}}
\label{app.lem2}

%\noindent\textbf{Sampling variance reduction via per-sample momentum.~} 
Here, we analyze the variance reduction effect of the per-sample momentum. 
Recall that for each sample \(\xi\), we compute a momentum \(v_t^{(\xi)}\) by 
averaging its gradients over the previous \(k\) iterations. This momentum is then
clipped and aggregated to form \(\bar{v}_t\), before applying the low-pass filter. 
We are interested in how closely \(v_t^{(\xi)}\) approximates the true 
gradient \(\nabla f(x_t)\) and, in particular, in the variance reduction achieved. 
We decompose the error as follows:
\begin{align*}
    & \mathbb{E}\left[\|v_t^{(\xi)} - \nabla f(x_t)\|^2\right] \\
   \leq & 2\underbrace{\mathbb{E}\left[\|v_t^{(\xi)} - \mathbb{E}[v_t^{(\xi)}]\|^2\right]}_{\text{sampling noise component}} + 2\underbrace{\|\mathbb{E}[v_t^{(\xi)}] - \nabla f(x_t)\|^2}_{\text{bias due to history}}.
\end{align*}
Under Assumptions~\ref{ass.var} (bounded variance) and~\ref{ass.ind} (independent sampling noise), the sampling noise component is reduced by a factor \(\rho^2\), where
\[
\rho = \sqrt{\frac{(1 + \beta)(1 - \beta^{k})}{(1-\beta)(1 + \beta^{k})}}.
\]
This result is formalized in the following lemma.

%We recall~\cref{lem.bmv}:
\bmv*

\begin{proof}
We begin by defining the sampling noise:
\[
\zeta_t^{(\xi)} = v_t^{(\xi)} - \nabla f(x_t).
\]
By expanding the definition of \(v_t^{(\xi)}\), we have:
\begin{footnotesize}
\begin{equation}
\label{eq:var_start}
\mathbb{E}\left[\|\zeta_t^{(\xi)}\|^2\right] 
=\mathbb{E}\left[\left\|\sum_{i=t-k+1}^{t} \frac{\beta^{t-i}}{c_\beta}\nabla f^{(\xi)}(x_i) - \nabla f(x_t)\right\|^2\right].
\end{equation}
\end{footnotesize}
We add and subtract \(\nabla f(x_i)\) inside the summation and gives:
\begin{footnotesize}
\begin{equation}\label{eq:decomp}
\mathbb{E}\Biggl[\Biggl\|\sum_{i=t-k+1}^{t} \frac{\beta^{t-i}}{c_\beta}\Bigl(\nabla f^{(\xi)}(x_i)-\nabla f(x_i) +\, \nabla f(x_i)-\nabla f(x_t)\Bigr)\Biggr\|^2\Biggr].
\end{equation}
\end{footnotesize}
Applying inequality $\|a+b\|^2 \le 2\|a\|^2 + 2\|b\|^2$ into~\Cref{eq:decomp}, we obtain:
\begin{footnotesize}
\begin{equation}\label{eq:decomp_terms}
\begin{aligned}
\mathbb{E}\left[\|\zeta_t^{(\xi)}\|^2\right] \le\,&2\,\mathbb{E}\Biggl[\Biggl\|\sum_{i=t-k+1}^{t} \frac{\beta^{t-i}}{c_\beta}\Bigl(\nabla f^{(\xi)}(x_i)-\nabla f(x_i)\Bigr)\Biggr\|^2\Biggr]\\[1mm]
&+2\,\mathbb{E}\Biggl[\Biggl\|\sum_{i=t-k+1}^{t} \frac{\beta^{t-i}}{c_\beta}\Bigl(\nabla f(x_i)-\nabla f(x_t)\Bigr)\Biggr\|^2\Biggr].
\end{aligned}
\end{equation}
\end{footnotesize}
We define:
\[
T_1 = \mathbb{E}\left[\left\|\sum_{i=t-k+1}^{t} \frac{\beta^{t-i}}{c_\beta}\Bigl(\nabla f^{(\xi)}(x_i)-\nabla f(x_i)\Bigr)\right\|^2\right].
\]
Under Assumption~\ref{ass.ind} and by linearity of expectation, we have:
\begin{footnotesize}
\begin{equation}\label{eq:bound_T1}
T_1 = \sum_{i=t-k+1}^{t} \frac{\beta^{2(t-i)}}{c_\beta^2}\mathbb{E}\left[\|\nabla f^{(\xi)}(x_i)-\nabla f(x_i)\|^2\right].
\end{equation}
\end{footnotesize}
Then, by using Assumption~\ref{ass.var}, $\mathbb{E}\left[\|\nabla f^{(\xi)}(x_i)-\nabla f(x_i)\|^2\right] \le \sigma_{SGD}^2$, we obtain:

\begin{equation}\label{eq:T1_bound}
T_1 \le \sigma_{SGD}^2 \sum_{i=t-k+1}^{t} \frac{\beta^{2(t-i)}}{c_\beta^2}.
\end{equation}
Noting that
\begin{footnotesize}
\[
c_\beta = \sum_{i=t-k+1}^{t}\beta^{t-i} = \frac{1-\beta^k}{1-\beta} \quad \text{and} \quad \sum_{i=t-k+1}^{t}\beta^{2(t-i)} = \frac{1-\beta^{2k}}{1-\beta^2},
\]
\end{footnotesize}
we obtain:
\begin{equation}\label{eq:T1_final}
T_1 \le \sigma_{SGD}^2\,\frac{\frac{1-\beta^{2k}}{1-\beta^2}}{\left(\frac{1-\beta^k}{1-\beta}\right)^2}.
\end{equation}
By defining $\rho = \sqrt{\frac{(1+\beta)(1-\beta^k)}{(1-\beta)(1+\beta^k)}}$, we have:
\begin{equation}\label{eq:rho}
\sum_{i=t-k+1}^{t} \frac{\beta^{2(t-i)}}{c_\beta^2} =\frac{\frac{1-\beta^{2k}}{1-\beta^2}}{\left(\frac{1-\beta^k}{1-\beta}\right)^2} = \frac{1}{\rho^2}.
\end{equation}
Hence,
\begin{equation}\label{eq:T1_final2}
T_1 \le \frac{\sigma_{SGD}^2}{\rho^2}.
\end{equation}
Next,  we define:
\[
T_2 = \mathbb{E}\left[\left\|\sum_{i=t-k+1}^{t} \frac{\beta^{t-i}}{c_\beta}\Bigl(\nabla f(x_i)-\nabla f(x_t)\Bigr)\right\|^2\right].
\]
Starting from the definition of \(L\)-smoothness property (Assumption~\ref{ass.lsmo}), we have:
\begin{equation}\label{eq:T2_step1}
T_2 \le L^2\,\mathbb{E}\left[\left\|\sum_{i=t-k+1}^{t} \frac{\beta^{t-i}}{c_\beta}(x_i-x_t)\right\|^2\right].
\end{equation}
Using the update rule \(x_{t+1}=x_t-\eta \hat{m}_t\) and applying Lemma~\ref{lem.trans}, we can express
\[
x_i-x_t = -\eta\sum_{l=i}^{t-1}\sum_{r=0}^l \hat{\kappa}_r \bar{v}_{l-r},
\]
so that
\begin{equation}\label{eq:T2_step2}
T_2 \le L^2\eta^2\,\mathbb{E}\left[\left\|\sum_{i=t-k+1}^{t} \frac{\beta^{t-i}}{c_\beta}\sum_{l=i}^{t-1}\sum_{r=0}^l \hat{\kappa}_r \bar{v}_{l-r}\right\|^2\right].
\end{equation}
Applying the Cauchy-Schwarz inequality to separate the summation over \(i\), we have:
\begin{footnotesize}
\begin{equation}\label{eq:T2_step3}
T_2 \le L^2\eta^2 \left(\sum_{i=t-k+1}^{t}\frac{\beta^{2(t-i)}}{c_\beta^2}\right)
\left(\sum_{i=t-k+1}^{t}\mathbb{E}\left[\left\|\sum_{l=i}^{t-1}\sum_{r=0}^l \hat{\kappa}_r \bar{v}_{l-r}\right\|^2\right]\right).
\end{equation}
\end{footnotesize}
Since the squared norm is convex, Jensen's inequality yields
\begin{equation}\label{eq:T2_step4}
\mathbb{E}\left[\left\|\sum_{l=i}^{t-1}\sum_{r=0}^l \hat{\kappa}_r \bar{v}_{l-r}\right\|^2\right] \leq k \sum_{l=i}^{t-1}\sum_{r=0}^l \hat{\kappa}_r\,\mathbb{E}\left[\|\bar{v}_{l-r}\|^2\right].
\end{equation}
As the DP noise is independent with the aggregated gradient, we obtain:
\[
\mathbb{E}\left[\|\bar{v}_{l-r}\|^2\right] \le C^2 + d\,\sigma_{DP}^2,
\]
and using \(\sum_{r=0}^l \hat{\kappa}_r = 1\), it follows that
\begin{footnotesize}
\begin{equation}\label{eq:T2_step5}
\begin{aligned}
\sum_{i=t-k+1}^{t}\mathbb{E}\left[\left\|\sum_{l=i}^{t-1}\sum_{r=0}^l \hat{\kappa}_r \bar{v}_{l-r}\right\|^2\right] \le \\[1mm] k(C^2 + d\,\sigma_{DP}^2)\sum_{i=t-k+1}^{t}(t-i).
\end{aligned}
\end{equation}
\end{footnotesize}

Since \(t-i\le k\) for \(i\in\{t-k+1,\dots,t\}\), combining \eqref{eq:T2_step3} and \eqref{eq:T2_step5} yields
\begin{equation}\label{eq:T2_final}
T_2 \le L^2\eta^2\,\left(\sum_{i=t-k+1}^{t}\frac{\beta^{2(t-i)}}{c_\beta^2}\right) k^3\,(C^2 + d\,\sigma_{DP}^2).
\end{equation}
Applying~\Cref{eq:rho} into~\Cref{eq:T2_final}, we obtain:
\begin{equation}\label{eq:T2_bound}
T_2 \le \frac{L^2\eta^2\, k^3\,(C^2 + d\,\sigma_{DP}^2)}{\rho^2}.
\end{equation}
Finally, by choosing the step size \(\eta\) to satisfy
\[
\eta \le \sqrt{\frac{\sigma_{SGD}^2}{L^2\,k^3\,(C^2 + d\,\sigma_{DP}^2)}},
\]
we ensure that
\begin{equation}\label{eq:T2_final1}
    T_2 \le \frac{\sigma_{SGD}^2}{\rho^2}.
\end{equation}
Substituting the bounds from \eqref{eq:T1_final2} and \eqref{eq:T2_final1} into \eqref{eq:decomp_terms}, we obtain
\[
\mathbb{E}\left[\|\zeta_t^{(\xi)}\|^2\right] \le 2\cdot\frac{\sigma_{SGD}^2}{\rho^2} + 2\cdot\frac{\sigma_{SGD}^2}{\rho^2}
= \frac{4\sigma_{SGD}^2}{\rho^2},
\]
which completes the proof.
\end{proof}

\section{Proof of Theorem~\ref{thm.con}}\label{app.thm1}
\con*
\begin{proof}
Firstly, by using the $L$-smoothness property from Assumption~\ref{ass.lsmo}, we have:
\begin{footnotesize}
\begin{align}\label{eqtot}
\mathbb{E}[f(x_{t+1}) - f(x_t)] \leq -\eta \mathbb{E}[\langle \nabla f(x_t), \hat{m}_t \rangle] + \frac{L\eta^2}{2} \mathbb{E}[\|\hat{m}_t\|^2].
\end{align}
\end{footnotesize}
Applying Lemma~\ref{lem.trans} to the term $-\eta \mathbb{E}[\langle \nabla f(x_t), \hat{m}_t \rangle]$, we have:
\begin{equation}\label{eqtot1}
    -\eta \mathbb{E}[\langle \nabla f(x_t), \hat{m}_t \rangle] = -\eta \mathbb{E}[\langle \nabla f(x_t), \sum_{r=0}^t \hat{\kappa}_r \bar{v}_{t-r} \rangle],
\end{equation}
then, using the linearity of expectation and inner product, we obtain:
\begin{equation}\label{eqtot1.s1}
-\eta \mathbb{E}[\langle \nabla f(x_t), \sum_{r=0}^t \hat{\kappa}_r \bar{v}_{t-r} \rangle] = -\eta \sum_{r=0}^t \hat{\kappa}_r \mathbb{E}[\langle \nabla f(x_t), \bar{v}_{t-r} \rangle].
\end{equation}
Expanding $\bar{v}_{t-r}$ yields:
\begin{align}\label{eqtot1.s2}
& -\eta \sum_{r=0}^t \hat{\kappa}_r \mathbb{E}[\langle \nabla f(x_t), \bar{v}_{t-r} \rangle] \notag \\
= & -\eta \sum_{r=0}^t \hat{\kappa}_r \mathbb{E}[\langle \nabla f(x_t), \frac{1}{B}\sum_{\xi \in \mathcal{B}_t} \tilde{v}_{t-r}^{(\xi)} + w_{t-r} \rangle].
\end{align}
Substituting~\Cref{eqtot1.s1} and~\Cref{eqtot1.s2} into~\Cref{eqtot1} and applying $\mathbb{E}[w_{t-r}] = 0$, we obtain:
\begin{equation}
    -\eta \mathbb{E}[\langle \nabla f(x_t), \hat{m}_t \rangle] = -\frac{\eta}{B} \sum_{r=0}^t \hat{\kappa}_r \sum_{\xi \in \mathcal{B}_t} \langle \nabla f(x_t), \mathbb{E}[\tilde{v}_{t-r}^{(\xi)}] \rangle.
\end{equation}
Adding and subtracting $\mathbb{E}[v_{t-r}^{(\xi)}]$ inside the summation, we have:
\begin{align}
& -\eta \mathbb{E}[\langle \nabla f(x_t), \hat{m}_t \rangle] \notag \\
=& -\frac{\eta}{B} \sum_{r=0}^t \hat{\kappa}_r \sum_{\xi \in \mathcal{B}_t} \langle \nabla f(x_t), \mathbb{E}[v_{t-r}^{(\xi)}] \rangle \notag \\
& - \frac{\eta}{B} \sum_{r=0}^t \hat{\kappa}_r \sum_{\xi \in \mathcal{B}_t} \langle \nabla f(x_t), \mathbb{E}[\tilde{v}_{t-r}^{(\xi)}-v_{t-r}^{(\xi)}] \rangle,
\end{align}
We start analyzing the first term $-\frac{\eta}{B} \sum_{r=0}^t \hat{\kappa}_r \sum_{\xi \in \mathcal{B}_t} \langle \nabla f(x_t), \mathbb{E}[v_{t-r}^{(\xi)}] \rangle$.
For $\langle \nabla f(x_t), \mathbb{E}[v_{t-r}^{(\xi)}] \rangle$, by expanding $v_{t-r}^{(\xi)}$, we have:
\begin{footnotesize}
\begin{equation}
\begin{aligned}
  \langle \nabla f(x_t), \mathbb{E}[v_{t-r}^{(\xi)}] \rangle = \left\langle \nabla f(x), \mathbb{E}\left[\sum_{i=t-r-k+1}^{t-r} \hat{\beta}^{t-r-i} \nabla f^{(\xi)}(x_i)\right] \right\rangle.
  \end{aligned}
\end{equation}
\end{footnotesize}

Using $\mathbb{E}[\nabla f^{(\xi)}(x_{t-r})] = \nabla f(x_{t-r})$, we obtain:
\begin{equation}
    \langle \nabla f(x_t), \mathbb{E}[v_{t-r}^{(\xi)}] \rangle = \sum_{i=t-r-k+1}^{t-r} \hat{\beta}^{t-r-i} \langle \nabla f(x_t), \nabla f(x_i) \rangle,
\end{equation}
then, applying Assumption~\ref{ass.corr} yields
\begin{align}
    &\langle \nabla f(x_t), \mathbb{E}[v_{t-r}^{(\xi)}] \rangle \notag\\
    \geq& \sum_{i=t-r-k+1}^{t-r} \hat{\beta}^{t-r-i} (c_{t-i}\|\nabla f(x_t)\|^2 + c_{-(t-i)}\|\nabla f(x_i)\|^2).
\end{align}
Therefore, we have
\begin{footnotesize}
\begin{align}\label{eq1}
     &-\eta \sum_{r=0}^t \hat{\kappa}_r \sum_{\xi \in \mathcal{B}_t} \langle \nabla f(x_t), \mathbb{E}[v_{t-r}^{(\xi)}] \rangle \notag \\
     \leq & -\eta \sum_{r=0}^t \hat{\kappa}_r \sum_{i=t-r-k+1}^{t-r} \hat{\beta}^{t-r-i} (c_{t-i}\|\nabla f(x_t)\|^2 + c_{-(t-i)}\|\nabla f(x_i)\|^2).
\end{align} 
\end{footnotesize}
For the second term $- \frac{\eta}{B} \sum_{r=0}^t \hat{\kappa}_r \sum_{\xi \in \mathcal{B}_t} \langle \nabla f(x_t), \mathbb{E}[\tilde{v}_{t-r}^{(\xi)}-v_{t-r}^{(\xi)}] \rangle$, we start by analyzing $ \|\mathbb{E}[\tilde{v}_{t-r}^{(\xi)}-v_{t-r}^{(\xi)}]\|$.
By using the definition of clipping, we have:
\begin{footnotesize}
\begin{equation}
    \|\mathbb{E}[\tilde{v}_{t-r}^{(\xi)}-v_{t-r}^{(\xi)}]\| = \left\|\mathbb{E}\left[\left(1-\frac{C}{\|v_{t-r}^{(\xi)}\|}\right)v_{t-r}^{(\xi)} \cdot \mathbf{1}_{\{\|v_{t-r}^{(\xi)}\|>C\}}\right]\right\|.
\end{equation}
\end{footnotesize}

Utilizing $(1-\frac{C}{\|v_{t-r}^{(\xi)}\|}) \leq 1$ and Jensen's Inequality, we obtain:
\begin{equation}
    \|\mathbb{E}[\tilde{v}_{t-r}^{(\xi)}-v_{t-r}^{(\xi)}]\| \leq \mathbb{E}[\|v_{t-r}^{(\xi)}\| \cdot \mathbf{1}_{\{\|v_{t-r}^{(\xi)}\|>C\}}.
\end{equation}
Adding and subtracting $\nabla f(x_{t-r})$ yields:
\begin{footnotesize}
\begin{equation}
     \|\mathbb{E}[\tilde{v}_{t-r}^{(\xi)}-v_{t-r}^{(\xi)}]\| =  \mathbb{E}[\|v_{t-r}^{(\xi)}-\nabla f(x_{t-r})+\nabla f(x_{t-r})\| \cdot \mathbf{1}_{\{\|v_{t-r}^{(\xi)}\|>C\}}].
\end{equation}
\end{footnotesize}

Then, applying Cauchy-Schwarz Inequality and $P(\|v_{t-r}^{(\xi)}\| > C) \leq 1$, we have: 
\begin{footnotesize}
\begin{align}
    &\|\mathbb{E}[\tilde{v}_{t-r}^{(\xi)}-v_{t-r}^{(\xi)}]\| \notag\\
    \leq & \left(\|\nabla f(x_{t-r})\| + \sqrt{\mathbb{E}[\|v_{t-r}^{(\xi)}-\nabla f(x_{t-r})\|^2]}\right) \cdot P(\|v_{t-r}^{(\xi)}\| > C) \notag\\
    \leq & \|\nabla f(x_{t-r})\| + \sqrt{\mathbb{E}[\|v_{t-r}^{(\xi)}-\nabla f(x_{t-r})\|^2]}.
\end{align}
\end{footnotesize}

Next, utilizing Lemma~\ref{lem.bmv}, we obtain:
\begin{equation}\label{eq2.1}
     \|\mathbb{E}[\tilde{v}_{t-r}^{(\xi)}-v_{t-r}^{(\xi)}]\| \leq \|\nabla f(x_{t-r})\| + \frac{2\sigma_{SGD}}{\rho}.
\end{equation}
Then, after applying $\langle a,b\rangle \geq - \|a\|\|b|,$ we have:
\begin{align}\label{eq2i}
    &- \frac{\eta}{B} \sum_{r=0}^t \hat{\kappa}_r \sum_{\xi \in \mathcal{B}_t} \langle \nabla f(x_t), \mathbb{E}[\tilde{v}_{t-r}^{(\xi)}-v_{t-r}^{(\xi)}] \rangle \notag\\
    \leq & \eta \sum_{r=0}^t \hat{\kappa}_r  \|\nabla f(x_t)\| \|\mathbb{E}[\tilde{v}_{t-r}^{(\xi)}-v_{t-r}^{(\xi)}]\|.
\end{align}
Substituting~\Cref{eq2.1} into~\Cref{eq2i}, we obtain:
\begin{align}
     &- \frac{\eta}{B} \sum_{r=0}^t \hat{\kappa}_r \sum_{\xi \in \mathcal{B}_t} \langle \nabla f(x_t), \mathbb{E}[\tilde{v}_{t-r}^{(\xi)}-v_{t-r}^{(\xi)}] \rangle \notag\\
    \leq & \eta \sum_{r=0}^t \hat{\kappa}_r  \|\nabla f(x_t)\| \left(\|\nabla f(x_{t-r})\| + \frac{2\sigma_{SGD}}{\rho}\right)
\end{align}

Applying Young's Inequality~\citep{young1912}, we obtain the bound of the second term:
\begin{footnotesize}
\begin{align}\label{eq2}
    &- \frac{\eta}{B} \sum_{r=0}^t \hat{\kappa}_r \sum_{\xi \in \mathcal{B}_t} \langle \nabla f(x_t), \mathbb{E}[\tilde{v}_{t-r}^{(\xi)}-v_{t-r}^{(\xi)}] \rangle \notag\\
    ~~~~~~\leq & \frac{\eta}{2} \sum_{r=0}^t \hat{\kappa}_r \sum_{i=t-r-k+1}^{t-r} \hat{\beta}^{t-r-i} c_{t-i}\|\nabla f(x_t)\|^2 \notag \\
    & + \eta \sum_{r=0}^t \frac{\hat{\kappa}_r}{\sum_{i=t-r-k+1}^{t-r} \hat{\beta}^{t-r-i} c_{t-i}} \left(\|\nabla f(x_{t-r})\|^2 + \frac{4\sigma_{SGD}^2}{\rho^2}\right)
\end{align}
\end{footnotesize}

For simplicity, we use $\hat{c}_r$ to represent $\sum_{i=t-r-k+1}^{t-r} \hat{\beta}^{t-r-i} c_{t-i}$ and $\hat{c}_{-r}$ to represent $\sum_{i=t-r-k+1}^{t-r} \hat{\beta}^{t-r-i} c_{-(t-i)}$. Substituting (\ref{eq1}) and (\ref{eq2}) into (\ref{eqtot}), we have:
\begin{equation*}
\begin{aligned}
    &\mathbb{E}[f(x_{t+1}) - f(x_t)] \notag \\
    \leq & -\eta \mathbb{E}[\langle \nabla f(x_t), \hat{m}_t \rangle] + \frac{L\eta^2}{2} \mathbb{E}[\|\hat{m}_t\|^2] \notag\\
    \leq & -\frac{\eta}{2} \sum_{r=0}^t \hat{\kappa}_r \hat{c}_r\|\nabla f(x_t)\|^2 \notag\\
    &-\eta \sum_{r=0}^t \hat{\kappa}_r \sum_{i=t-r-k+1}^{t-r-i} \hat{\beta}^{t-r-i} c_{-(t-i)}\|\nabla f(x_i)\|^2 \notag\\
    &+ \eta\sum_{r=0}^t \hat{\kappa}_r (\frac{1}{\hat{c}_{r}} - \frac{c_{-r}}{c_\beta})\|\nabla f(x_{t-r})\|^2 \notag\\
    &+ \eta\sum_{r=0}^t  \frac{\hat{\kappa}_r}{\hat{c}_{r}}\frac{4\sigma_{SGD}^2}{\rho^2} + \frac{L\eta^2}{2} \mathbb{E}[\|\hat{m}_t\|^2].
    \end{aligned}
\end{equation*}
Dropping the negative term, we obtain:
\begin{align}
    &\mathbb{E}[f(x_{t+1}) - f(x_t)] \notag \\
    \leq& -\frac{\eta}{2} \sum_{r=0}^t \hat{\kappa}_r \hat{c}_r\|\nabla f(x_t)\|^2 \notag\\
    &+ \eta\sum_{r=0}^t \hat{\kappa}_r (\frac{1}{\hat{c}_{r}} - \frac{c_{-r}}{c_\beta})\|\nabla f(x_{t-r})\|^2 \notag\\
    &+ \eta\sum_{r=0}^t  \frac{\hat{\kappa}_r}{\hat{c}_{r}}\frac{4\sigma_{SGD}^2}{\rho^2} + \frac{L\eta^2}{2} \mathbb{E}[\|\hat{m}_t\|^2].
\end{align}
Next, utilizing Assumption~\ref{ass.bod} yields:
\begin{equation*}
\begin{aligned}
    &\mathbb{E}[f(x_{t+1}) - f(x_t)] \notag \\
     \leq & -\frac{\eta}{2} \sum_{r=0}^t \hat{\kappa}_r \hat{c}_r\|\nabla f(x_t)\|^2 \notag\\
     &+ \eta\sum_{r=0}^t \hat{\kappa}_r (\frac{1}{\hat{c}_{r}} - \frac{c_{-r}}{c_\beta})G^2 \notag \\
     & + \eta\sum_{r=0}^t  \frac{\hat{\kappa}_r}{\hat{c}_{r}}\frac{4\sigma_{SGD}^2}{\rho^2} + \frac{L\eta^2}{2} \mathbb{E}[\|\hat{m}_t\|^2].
\end{aligned}
\end{equation*}
Then, given the DP noise as independent for different iterations, we could bound the term $\mathbb{E}[\|\hat{m}_t\|^2]$ and have:
\begin{align}
    &\mathbb{E}[f(x_{t+1}) - f(x_t)] \notag \\
      \leq& -\frac{\eta}{2} \sum_{r=0}^t \hat{\kappa}_r \hat{c}_r\|\nabla f(x_t)\|^2 + \eta\sum_{r=0}^t \hat{\kappa}_r (\frac{1}{\hat{c}_{r}} - \frac{c_{-r}}{c_\beta})G^2 \notag \\
      & + \eta\sum_{r=0}^t  \frac{\hat{\kappa}_r}{\hat{c}_{r}}\frac{4\sigma_{SGD}^2}{\rho^2} + \frac{L\eta^2}{2} C^2 + \frac{L\eta^2}{2}\sum_{r=0}^t \hat{\kappa}_r^2\, d\,\sigma_{DP}^2.
\end{align}
Summing over $t = 0,\cdots, T-1$ and assuming after $T$ steps, the loss function $f$ reaches to its lower bound, which means $f(x_T) = f^*$. We obtain:
\begin{align}
    &f^* - f(x_0) \notag \\
      \leq&  -\frac{\eta}{2} \sum_{t=0}^{T-1}\sum_{r=0}^t \hat{\kappa}_r \hat{c}_r\|\nabla f(x_t)\|^2 \notag\\
      &+ \eta\sum_{t=0}^{T-1}\sum_{r=0}^t \hat{\kappa}_r (\frac{1}{\hat{c}_{r}} - \frac{c_{-r}}{c_\beta})G^2 \notag \\
      & + \eta\sum_{t=0}^{T-1}\sum_{r=0}^t  \frac{\hat{\kappa}_r}{\hat{c}_{r}}\frac{4\sigma_{SGD}^2}{\rho^2} + \sum_{t=0}^{T-1}\frac{L\eta^2}{2} C^2 \notag\\
      &+ \frac{L\eta^2}{2}\sum_{t=0}^{T-1}\sum_{r=0}^t \hat{\kappa}_r^2\, d\,\sigma_{DP}^2.
\end{align}
Rearranging the above equation yields
\begin{align}
    & \frac{\eta}{2} \sum_{t=0}^{T-1}\sum_{r=0}^t \hat{\kappa}_r \hat{c}_r\|\nabla f(x_t)\|^2 \notag \\
      \leq& f(x_0) - f^* + \eta\sum_{t=0}^{T-1}\sum_{r=0}^t \hat{\kappa}_r (\frac{1}{\hat{c}_{r}} - \frac{c_{-r}}{c_\beta})G^2 \notag \\
      & + \eta\sum_{t=0}^{T-1}\sum_{r=0}^t  \frac{\hat{\kappa}_r}{\hat{c}_{r}}\frac{4\sigma_{SGD}^2}{\rho^2} + \sum_{t=0}^{T-1}\frac{L\eta^2}{2} C^2 \notag\\
      &+ \frac{L\eta^2}{2}\sum_{t=0}^{T-1}\sum_{r=0}^t \hat{\kappa}_r^2\, d\,\sigma_{DP}^2.
\end{align}

\begin{figure*}[t!]
    \centering
    \includegraphics[width=1.0\textwidth]{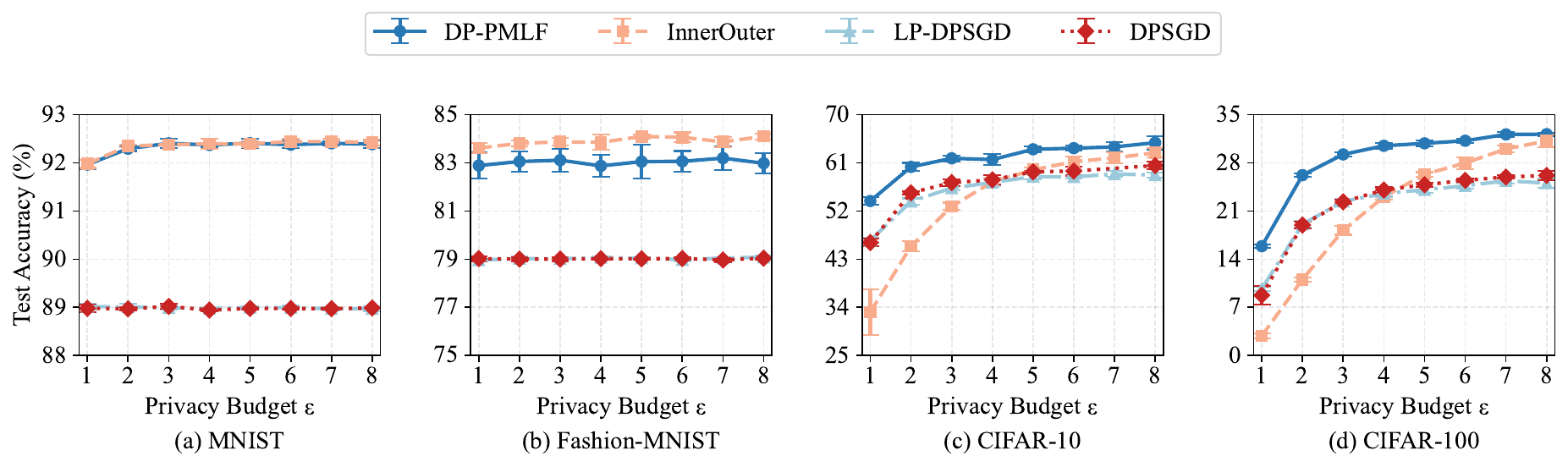}
    \caption{Test accuracy (\%) comparison across datasets on CNN-5 with fixed epoch (Epoch = 25 for MNIST and Fashion-MNIST, Epoch = 50 for CIFAR-10 and CIFAR-100) and different privacy budgets $\epsilon \in [1, 8].$}
    \label{fig.epsilon}
\end{figure*}

Finally, utilize $\sum_{t=0}^{T-1}\sum_{r=0}^t \hat{\kappa}_r \hat{c}_r = \mathcal{O}(T)$, let $\Gamma_{DP} = \frac{\sum_{t=0}^{T-1}\sum_{r=0}^t \hat{\kappa}_r \hat{c}_r}{\sum_{t=0}^{T-1}\sum_{r=0}^t \hat{\kappa}_r^2}$ and $\Gamma_{SGD} = \frac{\sum_{t=0}^{T-1}\sum_{r=0}^t \hat{\kappa}_r \hat{c}_r}{\sum_{t=0}^{T-1}\sum_{r=0}^t  \frac{\hat{\kappa}_r}{\hat{c}_{r}}}$, we obtain that the upper bound of 
$\mathbb{E}[\|\nabla f(x_t)\|^2]$ is
\begin{scriptsize}
$\mathcal{O}\left(\frac{f(x_0) - f^*}{\eta T}  + L\eta C^2 + \frac{L\eta}{\Gamma_{DP}}\, d\, \sigma^2_{DP} + \left(\frac{G^2}{\Gamma_{SGD}} +\frac{\sigma_{SGD}^2}{\rho^2 \, \Gamma_{SGD}}\right)\right).$    
\end{scriptsize}
\end{proof}
which completes the proof.
\section{Convergence Analysis of LP-DPSGD and DPSGD}\label{app.dd}
The main difference in convergence analysis between LP-DPSGD, vanilla DPSGD and our method DP-PMLF lies in the modification of the bound in Equation~\ref{eq2.1}. 

For LP-DPSGD, the bound in Equation~\ref{eq2.1} is adjusted to the term
\[
\|\nabla f(x_{t-r})\| + \sigma_{SGD},
\]
causing by the absence of per-sample momentum for variance reduction. The remainder of the proof remains the same, with an increased bias term in the convergence. Consequently, the upper bound on \(\mathbb{E}[\|\nabla f(x_t)\|^2]\) for LP-DPSGD is given by
\[
\mathcal{O}\left(\frac{f(x_0) - f^*}{\eta T}  + L\eta C^2 + \frac{L\eta}{\Gamma_{DP}}\, d\, \sigma^2_{DP} + \left(\frac{G^2 + \sigma_{SGD}^2}{\Gamma_{SGD}}\right)\right).
\]

In contrast, for vanilla DPSGD, the bound in Equation~\ref{eq2.1} is modified to yield an order of \(\mathcal{O}(\sigma_{SGD})\). Specifically, when \(\|\nabla f(x_t)\| \geq \sigma_{SGD}\), the magnitude of the true gradient is sufficient to compensate for the sampling variance, as detailed in Appendix C.2 of \citep{koloskova2023a}. The remainder of the proof remains unchanged, with the analysis omitting the low-pass filter coefficients. Therefore, the corresponding upper bound on \(\mathbb{E}[\|\nabla f(x_t)\|^2]\) for vanilla DPSGD is
\[
\mathcal{O}\left(\frac{f(x_0) - f^*}{\eta T} + L\eta C^2 + L\eta \, d\, \sigma^2_{DP} + \sigma_{SGD}^2\right).
\]

\section{Missing experiment details in the main paper}
\label{sec:expsetup}
In this section, we provide the missing details for the experiments in the main 
paper and additional experiments.
All methods are implemented using Python (version 3) and all experiments are run on
an i7-7800X CPU with a NVIDIA RTX A6000 GPU (48 GB). 
%For repeatability, we have made the code of our approach DP-PMLF available, which can be accessed through \url{https://anonymous.4open.science/r/DP-PMLF-521F}.%

\begin{table*}[t]
\centering
\caption{Test accuracy (\%) comparison across datasets on ResNet-18 with fixed epoch (Epoch = 25 for MNIST and Fashion-MNIST, Epoch = 50 for CIFAR-10 and CIFAR-100) and different privacy budgets $\epsilon = 1, 8$.}
\resizebox{\textwidth}{!}{%
\begin{tabular}{l|c|c|c|c|c|c|c|c}
\hline
Method & \multicolumn{2}{c|}{MNIST} & \multicolumn{2}{c|}{Fashion-MNIST} & \multicolumn{2}{c|}{CIFAR-10} & \multicolumn{2}{c}{CIFAR-100} \\
\cline{2-9}
& $\epsilon{=}1$ & $\epsilon{=}8$ & $\epsilon{=}1$ & $\epsilon{=}8$ & $\epsilon{=}1$ & $\epsilon{=}8$ & $\epsilon{=}1$ & $\epsilon{=}8$ \\
\hline
DPSGD & 89.01 ± 0.06 & 89.03 ± 0.04 & 79.05 ± 0.08 & 79.00 ± 0.08 & 41.61 ± 1.17 & 51.98 ± 0.55 & 5.79 ± 0.32 & 14.44 ± 0.39 \\
LP-DPSGD & 89.00 ± 0.05 & 89.02 ± 0.03 & 79.02 ± 0.11 & 79.02 ± 0.09 & 41.60 ± 1.02
 & 52.45 ± 1.29 & 6.11 ± 0.48 & 13.34 ± 0.09 \\
InnerOuter & 92.03 ± 0.09 & \textbf{92.46 ± 0.03} & \textbf{83.69 ± 0.15} & \textbf{83.93 ± 0.16} & 45.18 ± 0.66 & 57.65 ± 0.87 & 7.45 ± 0.30 & 21.57 ± 0.97 \\
DP-PMLF & \textbf{92.04 ± 0.04} & 92.38 ± 0.10 & 83.30 ± 0.35 & 83.50 ± 0.32 & \textbf{47.95 ± 0.91
} & \textbf{58.27 ± 1.52} & \textbf{11.49 ± 0.57} & \textbf{21.64 ± 0.38} \\
\hline
\end{tabular}
}
\label{tab.res}
\end{table*}

\subsection{Experimental Setup and Parameters}
\label{sec:addexp_setup}

\noindent\textbf{Dataset.} 
We evaluate our approach on four image classification datasets: 
MNIST~\citep{deng2012}, Fashion-MNIST~\citep{xiao2017}, 
CIFAR-10~\citep{krizhevsky2009}, and CIFAR-100~\citep{krizhevsky2009}. 
MNIST consists of 70,000 grayscale images of handwritten digits (0-9), with 60,000 
used for training and 10,000 for testing. Fashion-MNIST contains 70,000 grayscale
images of clothing items across 10 categories, with the same training and test 
split as MNIST. CIFAR-10 comprises 60,000 32×32 color images across 10 object 
categories, with 50,000 samples used for training and 10,000 used for testing. 
CIFAR-100 consists of 60,000 32×32 color images distributed across 100 categories, 
with the same training and test split as CIFAR-10. All dataset splits are directly
derived from the original definitions in prior work.

The details of the sentence classification datasets are shown as follows: MNLI consists of approximately 393,000 sentence pairs for training, where the task is to classify the relationship as entailment, neutral, or contradiction. QNLI is used to determine whether a sentence contains the answer to a question, with over 105,000 training pairs. QQP contains around 364,000 pairs of questions, where the goal is to determine if they are semantically equivalent. SST-2 is used for binary sentiment classification on movie reviews, providing over 67,000 training samples.

\begin{figure*}[t]
    \centering
    \includegraphics[width=0.8\textwidth]{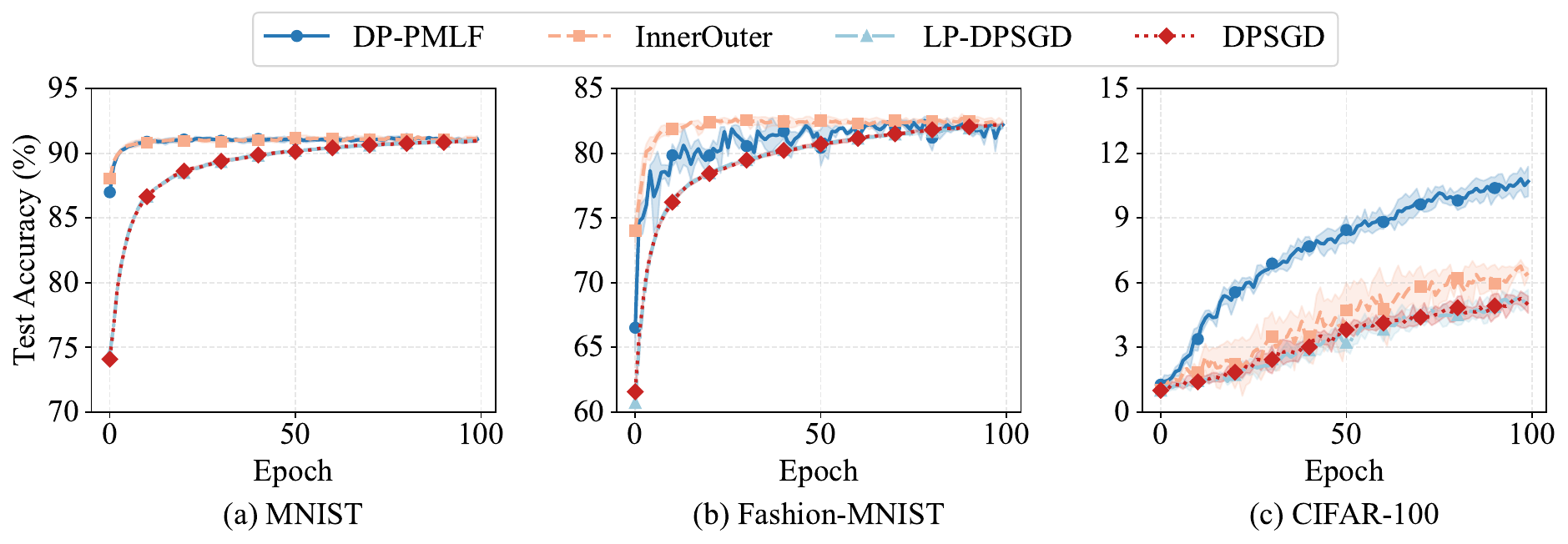}
    \caption{Test Accuracy (\%) Comparison Across Different Datasets on Resnet-18 with Fixed Privacy Budget $\epsilon =1.$}
    \label{fig.dataset}
\end{figure*}

%\vspace{0.2cm}
\noindent\textbf{Baselines.} 
We compare the test accuracy of our DP-PMLF method with vanilla 
DPSGD~\citep{abadi2016} and two state-of-the-art methods 
LP-DPSGD~\citep{zhang2024} and InnerOuter~\citep{xiao2023}. 
All these approaches can be applied directly to any deep learning models. 
We selected these baselines, as their designs are supported by theoretical
guarantees of convergence. We also notice that there can be some other DPSGD 
variants provide high test accuracy under some specific model parameters. 
However, a comparison with all those different DPSGD variants is beyond the
scope of this work and we have left such comparison for future work.   

%\vspace{0.2cm}
\noindent\textbf{Models.} 
In our experiments, we utilized three models: a 5-layer CNN~\citep{zhang2024}, ResNet-18~\citep{he2016}, and the Vision Transformer (ViT)~\citep{dosovitskiy2021}. 
these models are initialized with random weights without pretraining.
The CNN model comprises five convolutional layers, each followed by a Tanh activation
function and max-pooling operation. ResNet-18 is a residual network with 18 layers
incorporating shortcut connections. The ViT model applies the transformer to 
image classification by segmenting images into patches and processing them 
as sequences.

%\vspace{0.2cm}
\noindent\textbf{Hyper-parameters.}
The parameter choices are based on settings commonly used in the literature, 
as well as insights gained from preliminary experiments. %We also conduct hyper-parameter analysis in Section~\ref{sec.ps}. 

For all experiments, we set clipping threshold $C = 1$, learning rate 
$\eta = 0.5$, batch size $B = 1000$, and failure factor $\delta = \frac{1}{n}$, 
where $n$ is the number of samples in each dataset. We set the number of epochs 
to 25 for MNIST and Fashion-MNIST, 20 for CIFAR-10 and CIFAR-100, and 20 for all sentence classification datasets. 
For the hyper-parameters of per-sample momentum, we set the per-sample momentum 
factor $\beta = 0.1$, the momentum length $k = 2$. For the hyper-parameters of 
low-pass filtering, we set $a = \{-0.9\}, b = \{0.1\}$ for MNIST and 
Fashion-MNIST and $a = \{-0.9\}, b = \{0.15, -0.05\}$ for CIFAR-10 and 
CIFAR-100.

\subsection{Additional Experiments: Different datasets and models}
\label{sec:addexp_data}

%\vspace{0.2cm}
Figure~\ref{fig.epsilon} illustrates the trend of performance as $\epsilon$ ranges from 1 to 8 for the CNN-5 model, and we report the test accuracy for $\epsilon = 1$ and $\epsilon = 8$ for ResNet-18 in Table~\ref{tab.res}. 
%Although we performed experiments with $\epsilon < 1$, the results showed a trend similar to $\epsilon \geq 1$ and therefore due to space constraints, we omit them here.

As illustrated in Figure~\ref{fig.epsilon} for the CNN-5 model, the proposed DP-PMLF method consistently outperforms the baselines in most scenarios. For instance, on MNIST at \(\epsilon = 1\), DP-PMLF achieves approximately 92\% test accuracy, which is about 3\% higher than vanilla DP-SGD. As \(\epsilon\) increases to 8, its accuracy further improves to around 92.4\%. Likewise, on CIFAR-10, DP-PMLF attains an accuracy of approximately 54\% at \(\epsilon = 1\) (compared to about 46\% for the next-best baseline), and increases to nearly 65\% at \(\epsilon = 8\).

%The results for more complex models in 
Table~\ref{tab.res} further demonstrate the effectiveness of our approach on ResNet-18. 
%In Table~\ref{tab.res}, which presents the performance on ResNet-18, o
Our method outperforms vanilla DP-SGD by approximately 3-4\% on MNIST and Fashion-MNIST with both \(\epsilon = 1\) and \(\epsilon = 8\). On CIFAR-10 and CIFAR-100, the performance gains are around 6-7\%. 
%Similarly, the results in Table~\ref{tab.vit}, which focuses on ViT, confirm that DP-PMLF maintains its leading performance. For instance, on Fashion-MNIST, it achieves accuracies of about 80.65\% at \(\epsilon = 1\) and 81.93\% at \(\epsilon = 8\). On CIFAR-100, our approach maintains a 4--5\% margin over the next-best baselines across both privacy budgets.

However, on less complex datasets (such as MNIST and Fashion-MNIST) or when using models with fewer parameters (such as CNN-5 and Resnet-18), the performance of DP-PMLF is marginally worse than that of InnerOuter. For instance, there is a 0.39\% drop at \(\epsilon = 1\) and a 0.43\% drop at \(\epsilon = 8\) on Fashion-MNIST using Resnet-18 (Table~\ref{tab.res}). This is because the training iterations in those experimental scenarios are less affected by DP noise, resulting in gradients to be closely similar to those in a non-private setting. Thus, the normalization in per-sample momentum in DP-PMLF reduces gradient magnitude, leading to slower convergence under the same learning rate.

%Under a high DP noise regime ($\epsilon = 1$), as shown in Table~\ref{tab.vit}, the InnerOuter method shows a degradation in performance. This is because the InnerOuter method lacks normalization and suffers from excessive noise accumulation. In contrast, DP-PMLF utilizes both normalization and a low-pass filter allowing our approach to better control and filter DP noise, thereby achieving superior results.

As another set of experiments, we set the privacy budget $\epsilon=1$ and evaluate test accuracy over $T=100$ epochs. As shown in Figure~\ref{fig.dataset}, we compare the test accuracy of our approach on different datasets with the Resnet-18 model. In particular, DP-PMLF achieves higher test accuracy than the baselines on MNIST and CIFAR-100, reaching approximately 91\% and 11\%, respectively. Although its convergence speed on Fashion-MNIST is initially slower than that of InnerOuter, DP-PMLF ultimately converges to about 82.35\% accuracy at epoch 100, matching InnerOuter’s performance.

\subsection{Additional Experiments: Hyper-parameter Analysis}
\label{sec.ps}
%\subsection{\w{Hyper-parameter Analysis}}\label{sec.ps}
We examine hyper-parameters for the per-sample momentum and low-pass filter, providing empirical guidance for their selection. Figures~\ref{fig.beta},~\ref{fig.k}, and Table~\ref{tab.coef} illustrate the effects of per-sample momentum weight $\beta$, per-sample momentum length $k$, and low-pass filter coefficients ($a$ and $b$) on various datasets and privacy budgets, respectively.

As discussed in Section~\ref{sec.ca}, larger values of $\beta$ and $k$ increase the influence of historical gradients, reducing sampling variance and clipping bias but potentially slowing convergence. 
As illustrated in Figure~\ref{fig.beta}, for MNIST, the accuracy can even decrease for the same $\beta$ value as $\epsilon$ increases, particularly at $\epsilon = 8$. This phenomenon arises because of the assignment of excessive weight to historical gradients which potentially slow the convergence speed. 

\begin{figure}[t!]
    \centering
    \includegraphics[width=0.5\textwidth]{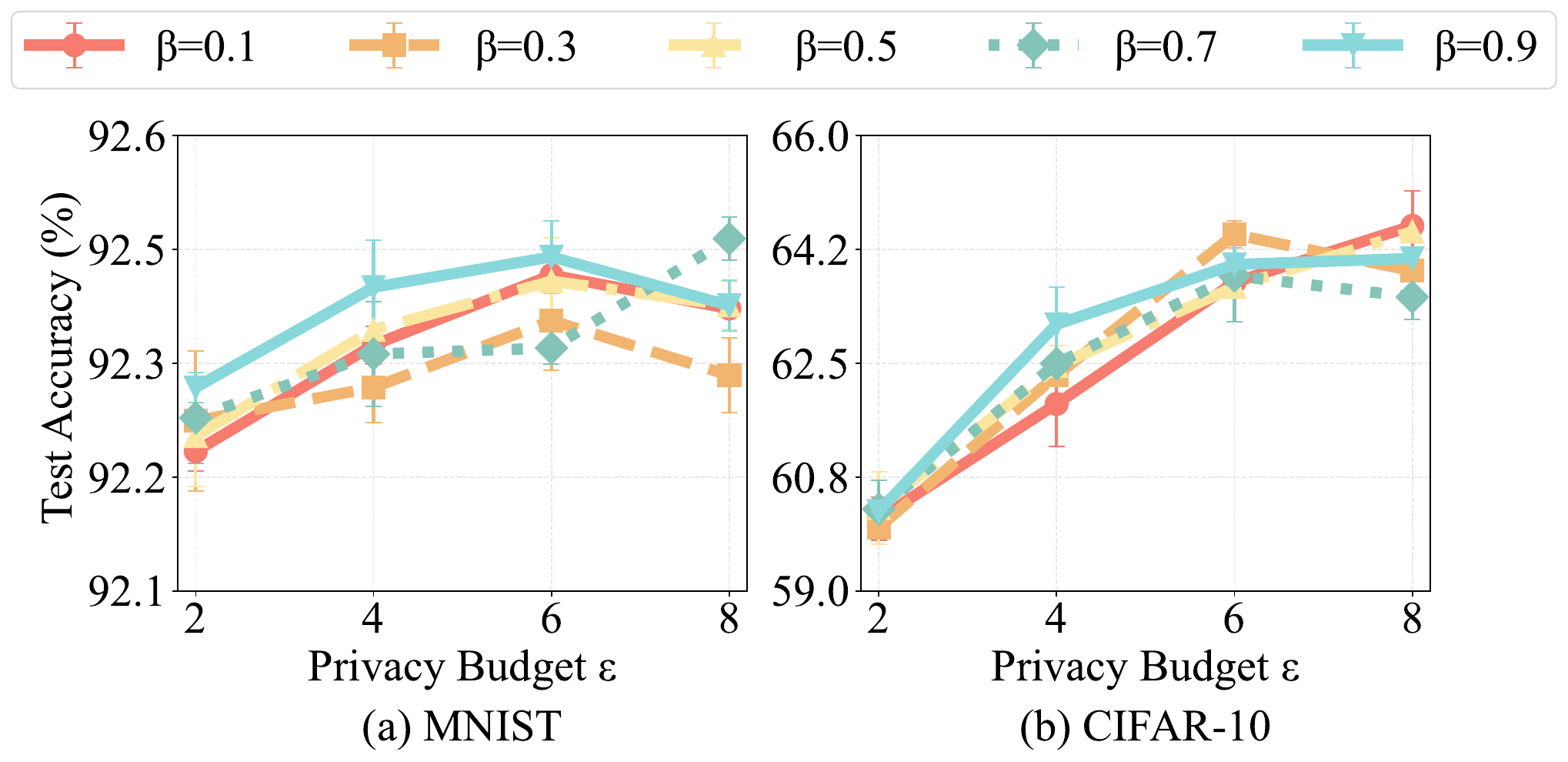}
    \caption{Test accuracy (\%) for different Per-sample Momentum weight $\beta = \{0.1, 0.3, 0.5, 0.7, 0.9\}$ with different privacy budgets $\epsilon = \{2, 4, 6, 8\}$ on MNIST and CIFAR-10 using CNN-5.}
    \label{fig.beta}
\end{figure}

However, in different $\epsilon$ values, the variation between different test accuracy values of different $\beta$ settings remains below 0.3\%, indicating that MNIST is relatively insensitive to the choice of $\beta$. For CIFAR-10, a smaller $\beta$ value is preferred when $\epsilon$ is high ($\epsilon = 8$). This allows our approach to focus more on current gradients during iterations. In contrast, a larger $\beta$ value is beneficial when $\epsilon$ is low ($\epsilon = 2$) which can balance the weights of all historical gradients to better mitigate the effects of DP noise.

Figure~\ref{fig.k} shows the performance across different datasets and privacy levels with $k$ ranging from 2 to 6. $k = 2$ yields the best performance for MNIST, while $k = 4$ is optimal for CIFAR-10. This shows that the selection of $k$ requires balancing the reduction of sampling variance with the reliance on historical gradients.

\begin{table*}[!t]
\centering
\caption{Test accuracy (\%) for eight different Low-pass Filter coefficients with different privacy budgets $\epsilon = \{2, 4, 6, 8\}$ on MNIST and CIFAR-10 using CNN-5. In the last row, the filter coefficient indicates there is no Low-pass Filter.}
\label{tab.coef}
\resizebox{\textwidth}{!}{%
\begin{tabular}{l|cccc|cccc}
\hline
 & \multicolumn{4}{c|}{MNIST} & \multicolumn{4}{c}{CIFAR-10}\\
\cline{2-9}
\textbf{Filter Coefficients} & $\epsilon=2$ & $\epsilon=4$ & $\epsilon=6$ & $\epsilon=8$
                & $\epsilon=2$ & $\epsilon=4$ & $\epsilon=6$ & $\epsilon=8$\\
\hline
\begin{tabular}[c]{@{}l@{}}a:\{\,-0.6\,\}, b:\{0.4\}\end{tabular}& 
91.48$\pm$0.03 & 91.59$\pm$0.07 & 91.62$\pm$0.05 & 91.54$\pm$0.02 &
58.40$\pm$1.27 & 63.19$\pm$0.29 & \textbf{64.43$\pm$0.16} & \textbf{65.69$\pm$0.34} \\
\hline
\begin{tabular}[c]{@{}l@{}}a:\{\,-0.6\,\}, b:\{0.2, 0.2\}\end{tabular}& 
90.63$\pm$0.01 & 90.61$\pm$0.01 & 90.60$\pm$0.03 & 90.56$\pm$0.06 &
59.16$\pm$0.85 & 63.23$\pm$0.53 & 63.68$\pm$0.20 & 65.10$\pm$0.91 \\
\hline
\begin{tabular}[c]{@{}l@{}}a:\{\,-0.6\,\}, b:\{0.5, -0.1\}\end{tabular}& 
90.61$\pm$0.04 & 90.60$\pm$0.03 & 90.57$\pm$0.02 & 90.56$\pm$0.02 &
60.12$\pm$0.42 & \textbf{63.48$\pm$0.48} & 63.94$\pm$1.38 & 64.76$\pm$0.80 \\
\hline
\begin{tabular}[c]{@{}l@{}}a:\{\,-0.9\,\}, b:\{0.1\}\end{tabular}& 
\textbf{92.38$\pm$0.05} & \textbf{92.49$\pm$0.01} & \textbf{92.36$\pm$0.07} & \textbf{92.45$\pm$0.06} &
57.65$\pm$0.65 & 58.71$\pm$0.86 & 62.35$\pm$0.55 & 63.58$\pm$0.39 \\
\hline
\begin{tabular}[c]{@{}l@{}}a:\{\,-0.9\,\}, b:\{0.05, 0.05\}\end{tabular}& 
90.58$\pm$0.05 & 90.62$\pm$0.02 & 90.61$\pm$0.01 & 90.63$\pm$0.01 &
58.02$\pm$0.60 & 62.02$\pm$0.66 & 62.89$\pm$1.34 & 63.21$\pm$0.38 \\
\hline
\begin{tabular}[c]{@{}l@{}}a:\{\,-0.9\,\}\, b:\{0.15, -0.05\}\end{tabular}& 
90.59$\pm$0.05 & 90.61$\pm$0.02 & 90.56$\pm$0.03 & 90.64$\pm$0.03 &
\textbf{60.47$\pm$0.59} & 62.55$\pm$0.68 & 63.86$\pm$0.57 & 64.58$\pm$0.08 \\
\hline
\begin{tabular}[c]{@{}l@{}}a:\{\,-0.7, -0.2\,\}, b:\{0.05, 0.05\}\end{tabular}& 
90.66$\pm$0.03 & 90.60$\pm$0.04 & 90.61$\pm$0.03 & 90.60$\pm$0.01 &
57.74$\pm$0.56 & 62.14$\pm$0.90 & 62.20$\pm$0.94 & 62.97$\pm$0.66 \\
\hline
\begin{tabular}[c]{@{}l@{}}a:\{\,\,\}, b:\{1\}\end{tabular} & 
90.60$\pm$0.02 & 90.63$\pm$0.02 & 90.57$\pm$0.02 & 90.59$\pm$0.04 &
59.91$\pm$0.88 & 62.05$\pm$0.46 & 63.77$\pm$0.66 & 64.57$\pm$0.68 \\
\hline
\end{tabular}
}
\end{table*}

Table~\ref{tab.coef} shows the comparison among eight different coefficient settings. All settings satisfy the constraint defined in~\Cref{eq.constraint}. For MNIST, the setting with $a:\{-0.9\}, b:\{0.1\}$ consistently outperforms others, achieving approximately 92.4\% accuracy across all privacy budgets. This is because the filter provides a relatively accurate gradient descent direction with strong historical dependence to accelerate convergence.
\begin{figure}[t!]
    \centering
    \includegraphics[width=0.5\textwidth]{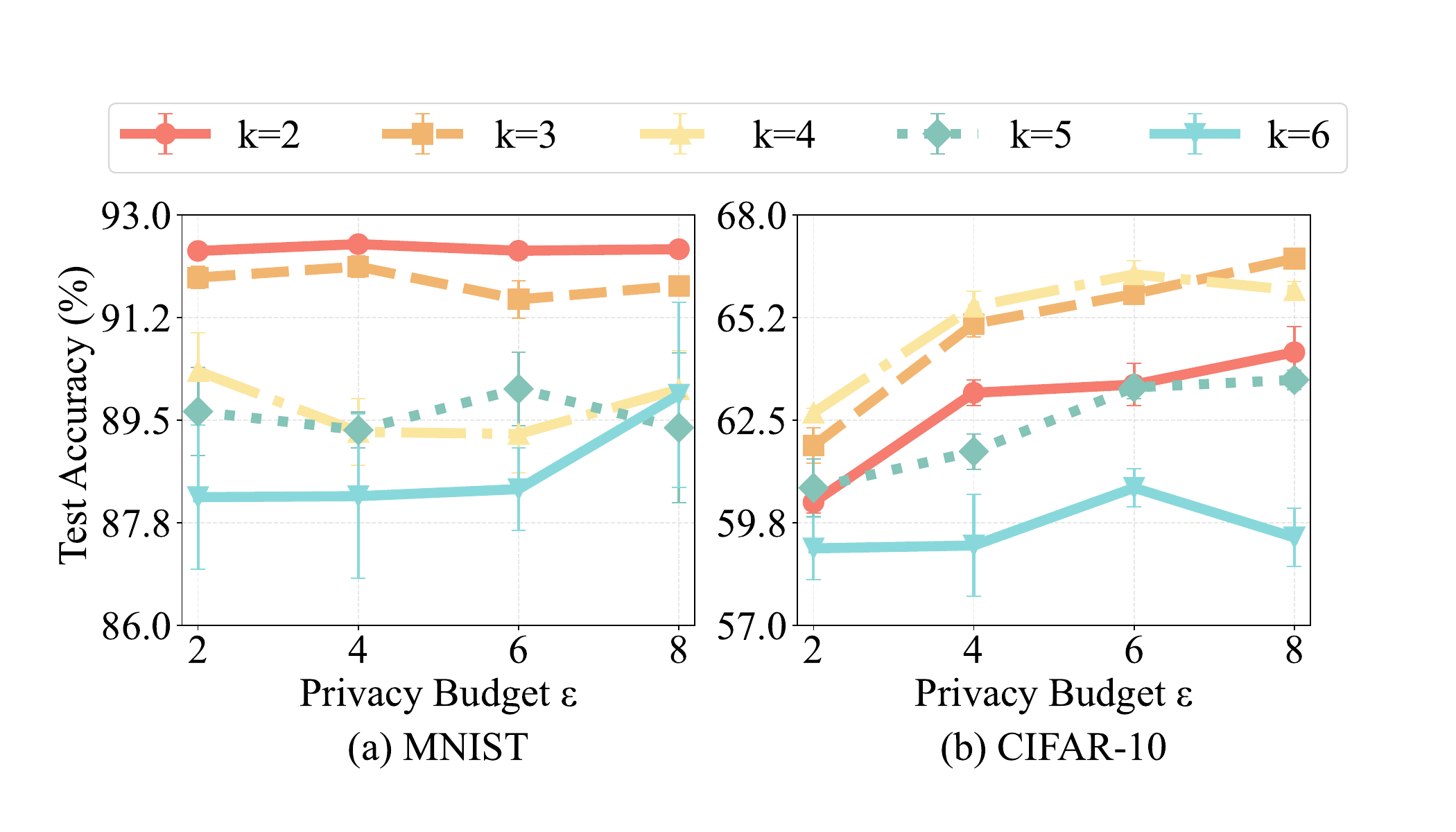}
    \caption{Test accuracy (\%) for different momentum length $k = \{2, 3, 4, 5, 6\}$ with different privacy budgets $\epsilon = \{2, 4, 6, 8\}$ on MNIST and CIFAR-10 using CNN-5.}
    \label{fig.k}
\end{figure}

In contrast, for CIFAR-10, the optimal setting varies with different values of the privacy budget $\epsilon$. Under high DP noise regimes (low $\epsilon$ values, $\epsilon\le 4$), second-order filters with mixed-sign coefficients, particularly $a:\{-0.6\}, b:\{0.5, -0.1\}$ and $a:\{-0.9\}, b:\{0.15, -0.05\}$, demonstrate superior performance. These settings create more selective frequency responses that better match the gradient information of complex image datasets, thus improving the signal-to-noise ratio. However, all settings where $b$ values are positive (such as $a:\{-0.9\}, b:\{0.05, 0.05\}$) perform consistently worse on CIFAR-10, as they over-smooth the gradient signal and eliminate potentially useful information.

For higher privacy budgets ($\epsilon = 8$), the setting $a:\{-0.6\}, b:\{0.4\}$ achieves the best performance on CIFAR-10 with 65.69\% accuracy, suggesting that when DP noise is less, a simpler first-order filter with moderate historical dependency provides an optimal balance between noise suppression and preservation of gradient signals. 

\section{Notations}
Table~\ref{tab.n} summarizes the notation, divided into three parts: basic symbols, algorithm-specific terms, and additional notations for our theoretical analysis.
\begin{table}[h!]
\caption{Summary of Notations}
\label{tab.n}
\resizebox{\columnwidth}{!}{%
\begin{tabular}{cl}
\toprule
\textbf{Notation} & \textbf{Description} \\
\midrule
$D = \{\xi_i\}_{i=1}^n$ & Dataset \\
$x$ & Model parameters \\
$\eta$ & Learning rate \\
$B$ & Batch size \\
$T$ & Total number of iterations \\
$\epsilon$ & Privacy budget\\
$\delta$ & Failure factor\\
$C$ & Clipping threshold \\
$\sigma_{DP}$ & Standard variance of DP noise \\
$d$ & Dimension of model parameters \\
\midrule
$\beta$ & Per-sample momentum weight \\
$c_\beta$ & Momentum normalization constant \\
$k$ & Per-sample momentum length \\
$a = \{a_r\}_{r=1}^{n_a}$ & Low-pass filter parameters for historical terms \\
$b = \{b_r\}_{r=0}^{n_b}$ & Low-pass filter parameters for momentum terms \\
$v_t^{(\xi)}$ & Per-sample momentum for sample $\xi$ at iteration $t$ \\
$\tilde{v}_t^{(\xi)}$ & Clipped per-sample momentum \\
$\bar{v}_t$ & Averaged (and noised) momentum across batch \\
$m_t$ & Filtered momentum \\
$\hat{m}_t$ & Normalized filtered momentum \\
$c_{m,t}$ & Filter normalization term at iteration $t$\\
\midrule
$L$ & Lipschitz constant of smoothness\\
$\sigma_{SGD}$ & Bounded sampling variance\\
$G$ & Bounded gradient norm \\
$\{c_r\},\{c_{-r}\}$ & Correlation coefficients \\
$\hat{c}_r$ & Correlation coefficient after per-sample momentum \\
$\zeta_i^{(\xi)}$ & Sampling noise for sample $\xi$ at iteration $t$ \\
$\kappa_r$ & Filter coefficient for each momentum \\
$\hat{\kappa}_r$ & Normalized filter coefficient for each momentum \\
$\rho$ & Reduced sampling variance ratio \\
$f^*$ & Lower bound of the ERM problem\\
\bottomrule
\end{tabular}
}
\end{table}
\end{document}